\documentclass{llncs}
\usepackage{makeidx}  
\usepackage{amsmath}
\usepackage{amssymb}
\usepackage{graphicx}
\usepackage{caption} 
\usepackage{subfig}
\usepackage{verbatim}
\DeclareMathOperator*{\argmin}{arg\,min}
\DeclareMathOperator*{\argmax}{arg\,max}
\usepackage{algorithm}
\usepackage{algorithmic}
\usepackage{multirow}

\begin{document}
\mainmatter              

\title{Fast and Bounded Probabilistic Collision Detection in Dynamic Environments for High-DOF Trajectory Planning}

\author{Chonhyon Park \and Jae Sung Park \and Dinesh Manocha}
\institute{University of North Carolina, Chapel Hill, NC 27599, USA,\\
\email{\{chpark,jaesungp,dm\}@cs.unc.edu}}

\maketitle

\begin{abstract}

We present a novel approach to perform probabilistic collision detection between a high-DOF robot and high-DOF obstacles in dynamic, uncertain environments.
In dynamic environments with a high-DOF robot and moving obstacles, our approach efficiently computes accurate collision probability between the robot and obstacles with upper error bounds.
Furthermore, we describe a prediction algorithm for future obstacle position and motion that accounts for both spatial and temporal uncertainties.
We present a trajectory optimization algorithm for high-DOF robots in dynamic, uncertain environments based on probabilistic collision detection.
We highlight motion planning performance in challenging scenarios with robot arms operating in environments with dynamically moving human obstacles.

\end{abstract}

\section{Introduction}

Robots are increasingly being used in living spaces, factories, and outdoor environments. One recent trend has been the development of co-robots (or cobots), robots that are intended to physically interact with humans in a shared workspace. In such environments, various elements or parts of the robot tend to be in close proximity to the humans or other moving objects. 
This proximity gives rise to two kinds of challenges in terms of motion planning. 
First, we have to predict the future actions and reactions of moving obstacles or agents in the environment to avoid collisions with obstacles. 
Therefore, the collision avoidance algorithm needs to deal with uncertain and imperfect representation of future obstacle motions efficiently.
Second, the computed robot motion still needs to be reasonably efficient. It is not desired to compute a very slow or excessively diverting trajectory in order to avoid collisions.

Various uncertainties arise from control errors, sensing errors, or environmental errors (i.e. imperfect environment representation) in the estimation and prediction of environment obstacles.
Typically, these uncertainties are modeled using Gaussian distributions.
Motion planning algorithms use probabilistic collision detection to avoid collisions with the given imperfect obstacle representation.
With such obstacle representations, it can be impossible (Gaussian distributions of obstacle positions have non-zero probabilities in the entire workspace) to compute a perfectly collision-free, or results in an inefficient trajectory to avoid collisions with very low probabilities.
In order to balance the safety and efficiency of planned motions, motion planning under uncertainties is desired to guarantee collision-free of the computed trajectory only in a limited probabiliy bound, which can be specified using a parameter, \emph{confidence level} (e.g. 0.99)~\cite{du2012robot}.

For the probabilistic collision detection, stochastic algorithms are used to approximate the collision probability \cite{blackmore2006probabilistic,lambert2008fast}.
However, such probabilistic collision detection algorithms are computationally intensive and mostly limited to 2D spaces.
Most prior planning approaches for high-DOF robots perform the exact collision checking with scaled objects that enclose the potential object volumes~\cite{bry2011rapidly,van2012lqg,lee2013sigma,sun2015stochastic}. Although these approaches guarantee probabilistical safety bounds, they highly overestimate the collision probability, which result in less optimal trajectories or failure to finding feasible trajectories in a limited planning time in dynamic environments.

\noindent {\bf Main Results:} 
In this paper, we present a novel approach to perform probabilistic collision detection.
Our approach has two novel contributions. First, we provide a fast approximation of collision probability between the high-DOF robot and high-DOF obstacles. Our approach computes more accurate probabilities than approaches using exact collision checking with enlarged obstacle shapes, and the computed probability is guaranteed as the upper bound that the actual probability is always lower than the computed probability.
Second, we describe a practical belief space estimation algorithm that accounts for both spatial and temporal uncertainties in the position and motion of each obstacle in dynamic environments with moving obstacles.

We present a trajectory optimization algorithm for high-DOF robots in dynamic, uncertain environments based on our probabilistic collision detection and belief space estimation.
We have evaluated our planner using robot arms operating in a simulation and a real environment workspace with high-resolution point cloud data corresponding to moving human obstacles, captured using a Kinect. 
Our approach uses a high value of the confidence level ($0.95$ or above) to perform probabilistic collision detection and can compute a smooth collision-free trajectory.

The paper is organized as follows.
Section~\ref{sec:related} gives a brief overview of prior work on probabilistic collision detection and motion planning.
We introduce the notation and the algorithm of our probabilistic collision detection algorithm in Section~\ref{sec:pcc}.
We describe the belief space estimation and trajectory planning algorithm in Section~\ref{sec:environment} and Section~\ref{sec:optimization}, respectively. 
We highlight planning performance in challenging human environment scenarios in Section~\ref{sec:result}.

\section{Related Work}
\label{sec:related}

In this section, we give a brief overview of prior work on probabilistic collision detection, trajectory planning, and uncertainty handling.

\subsection{Probabilistic Collision Detection}

Collision checking is an integral part of any motion planning algorithm and most prior techniques assume an exact representation of the robot and obstacles. Given some uncertainty or imperfect representation of the obstacles, the resulting algorithms perform probabilistic collision detection.
Typically, these uncertainties are modeled using Gaussian distributions and stochastic algorithms are used to approximate the collision probability \cite{blackmore2006probabilistic,lambert2008fast}.
In stochastic algorithms, a large number of sample evaluations are required to compute the accurate collision probability.

If it can be assumed that the size of the objects is relatively small, the collision probability can be approximated using the probability at a single configuration corresponds to the mean of the probability distribution function (PDF), which provides a closed-form solution~\cite{du2011probabilistic}. This approximation is fast, but the computed probability cannot provide a bound, and can be either higher or lower than the actual probability, where the error increases as the object is bigger and has high-DOFs.

For high-dimensional spaces, a common approach for checking collisions for imperfect or noisy objects is to perform the exact collision checking with scaled objects that enclose the potential object volumes~\cite{van2012lqg,Park:2012:ICAPS}. Prior approaches generally enlarge an object shape, which may correspond to a robot or an obstacle, to compute the space occupied by the object for a given standard deviation. This may correspond to a sphere~\cite{bry2011rapidly} or a ‘sigma hull’~\cite{lee2013sigma}.
This approach provides an upper bounding volume for the given confidence level. However, the computed volume overestimates the probability and can be much bigger than the actual volume corresponds to the confidence level, which can cause failure of finding existing feasible trajectories in motion planning.

Many other approaches have been proposed to perform probabilistic collision detection on point cloud data.
Bae et al.~\cite{bae2009closed} presented a closed-form expression for the positional uncertainty of point clouds.
Pan et al.~\cite{pan2011probabilistic} reformulate the probabilistic collision detection problem as a classification problem and compute per point collision probability. However, these approaches assume that the environment is static. Other techniques are based on broad phase data structures that handle large point clouds for realtime collision detection~\cite{pan2013real}.

\subsection{Planning in Dynamic Environments}
There is considerable literature on motion planning in dynamic scenes. In some applications, the future locations or trajectories of the obstacles are known. As a result, the time dimension is added to the configuration space of the robot and classical motion planning algorithms can be applied to the resulting state space~\cite{LaValle:2006}.
In many scenarios, the future positions of the obstacles are not known. As a result, the planning problem is typically solved locally using reactive techniques such as dynamic windows or velocity obstacles~\cite{Fiorini:1998}, or assuming that the obstacle trajectories are known within a short horizon~\cite{Likhachev:2009}.
Other methods use replanning algorithms, which interleave planning with execution. These methods include sampling-based planners~\cite{Hauser:safety,David:2002,SMP:2005}, grid searches~\cite{Koenig:2003:PBP,Likhachev05anytimedynamic}, or trajectory optimization~\cite{Park:2012:ICAPS}. Our formulation is based on optimization-based replanning, and we  take into account smoothness and dynamic constraints.

Applications that require high responsiveness use control-based approaches~\cite{haschke2008line,kroger2010online}, which can compute trajectories in realtime.
They compute the robot trajectory in the Cartesian space, i.e. the workspace of the robot, according to the sensor data. However, the mapping from the Cartesian trajectory to the trajectory in the configuration space of high-DOF robots can be problematic. Furthermore, control-based approaches tend to compute less optimal robot trajectories as compared to the planning approaches that incorporate the estimation of the future obstacle poses. Planning algorithms can compute better robot trajectories in applications in which a good prediction about obstacle motions in a short horizon can be provided.

\subsection{Planning under Uncertainties}

The problem of motion planning under uncertainty, or belief space planning, has been an active area of research for the last few decades. The main goal is to plan a path for a robot in partially-observable state spaces. The underlying problem is formally defined using  POMDPs (partially-observable Markov decision processes), which provide a mathematically rigorous and general approach for planning under uncertainty~\cite{kaelbling1998planning}.
The resulting POMDP planners handle the uncertainty by reasoning over the {\em belief} space. A belief corresponds to the probability distribution over all possible states. However, The POMDP formulation is regarded as computationally intractable~\cite{papadimitriou1987complexity} for problems which are high-dimentional or have a large number of actions. 
Therefore, many efficient approximations~\cite{silver2010monte,kurniawati2013online,somani2013despot} and parallel techniques~\cite{shani2010evaluating,lee2013gpu} have been proposed to provide a better estimation of belief space.

Most approaches for continuous state spaces use Gaussian belief spaces, which are estimated using Bayesian filters (e.g., Kalman filters)~\cite{leung2006planning,platt2010belief}.
Algorithms using Gaussian belief spaces have also been proposed for the motion planning of high-DOF robots~\cite{van2012lqg,sun2015stochastic}, but they do not account for environment uncertainty or imperfect obstacle information.
Instead, most planning algorithms handling environment uncertainty deal with issues arising from visual occlusions~\cite{missiuro2006adapting,guibas2010bounded,kahn2015active,charrow2015information}. In terms of dynamic environments, motion planning with uncertainty algorithms is mainly limited to 
2D spaces~\cite{du2012robot,bai2015intention}, where the robots are modeled as circles, or to specialized applications such as people tracking~\cite{bandyopadhyay2009motion}. 

\section{Probabilistic Collision Detection}
\label{sec:pcc}

In this section, we first introduce the notation and terminology used in the paper and present our probabilistic collision checking algorithm between the robot and the environment.

\subsection{Notation and Assumptions}
\label{subsec:notation}

Our goal is to compute a collision probability between a high-DOF robot configuration and a given obstacle representation of dynamic environments, where the obstacle representation is a probability distribution which accounts uncertainties in the future obstacle motion prediction.

For an articulated robot with $D$ one-dimensional  joints, we represent a single robot configuration as $\mathbf q$, which is a vector composed from the joint values.
The $D$-dimensional vector space of $\mathbf q$ is the configuration space $\mathcal{C}$ of the robot.
We denote the collision-free subset of $\mathcal{C}$ as $\mathcal{C}_{free}$, and the other configurations corresponding to collisions as $\mathcal{C}_{obs}$.

We assume that the robot consists of $J$ links  $R_1,...,R_J$, where $J \leq D$. 
Furthermore, for each robot link $R_j$, we use multiple bounding volumes $B_{j1},...,B_{jK}$ to tightly enclose $R_j(\mathbf q)$ which corresponds to a robot configuration $\mathbf q$, i.e., 
\begin{equation}
\label{eq:robot_bv}
\begin{split}
\forall j : R_j(\mathbf q) \subset \bigcup_{k=1}^K B_{jk}(\mathbf q) \,\,\textrm{for}\,\, (1 \le j \le J).
\end{split}
\end{equation}
In our experiments, bounding spheres are automatically generated along the medial axis of each robot link.

We represent $L$ obstacles in the environment as $O_l \, (1 \le l \le L)$, and assume that the obstacles undergo rigid motion. The configuration of these obstacles is specified based on geometric (shape) representation and their poses. 
As is the case for the robot, we use the bounding volumes $S_{l1},...,S_{lM}$ to enclose each obstacle $O_l$ in the environment:
\begin{equation}
\label{eq:obs_bv}
\begin{split}
\forall l : O_l \subset \bigcup_{m=1}^M S_{lm} \,\,\textrm{for}\,\, (1 \le l \le L).
\end{split}
\end{equation}
For dynamic obstacles, we assume the predicted position of a bounding volume $S_{lm}$ at time $t$ is estimated as a Gaussian distribution $\mathcal{N} (\mathbf p_{lm}, \mathbf \Sigma_{lm})$, which will be described in Section~\ref{sec:environment}.

\subsection{Probabilistic Collision Checking}
\label{subsec:pcc}

The collision probability between a robot configuration $\mathbf q_i$ with the environment at time $t_i$, $P(\mathbf q_i \in \mathcal{C}_{obs}(t_i))$ can be formulated as
\begin{equation}
\label{eq:colspace}
\begin{split}
P\left(\left(\bigcup_j \bigcup_k  B_{jk}(\mathbf q_i)\right) \bigcap \left(\bigcup_l \bigcup_m S_{lm}(t_i) \right) \neq \emptyset \right).
\end{split}
\end{equation}
We assume the robot links $R_j$ and obstacles $O_l$ are independent with each other link or obstacle, as their positions depend on corresponding joint values or obstacle states. Then (\ref{eq:colspace}) can be computed as
\begin{align}
\label{eq:colprob}
&P(\mathbf q_i \in \mathcal C_{obs}(t_i)) =1-\prod_j\prod_l \overline{P_{col}(i,j,l) },
\end{align}
where $P_{col}(i,j,l)$ is the collision probability between $R_j(\mathbf q_i)$ and obstacles $O_l(t_i)$. Since positions of bounding volumes $B_{jk}$ and  $S_{lm}$ are determined by joint values or obstacle states of the corresponding robot link or obstacle, bounding volumes for the same object are dependant with each other, and $P_{col}(i,j,l)$ can be approximated as
\begin{align}
\label{eq:colprob2}
&P_{col}(i,j,l)\approx\max_{k,m} P_{col}(i,j,k,l,m)\\
&P_{col}(i,j,k,l,m)=P(B_{jk}(\mathbf q_i)\cap S_{lm}(t_i) \neq \emptyset),
\end{align}
where $P_{col}(i,j,k,l,m)$ denotes the collision probability between $B_{jk}(\mathbf q_i)$ and $S_{lm}(t_i)$.

\begin{figure}[t]
  \centering
  \includegraphics[trim=0in 0in 0in 1.0in, clip=true, width=0.6\linewidth]{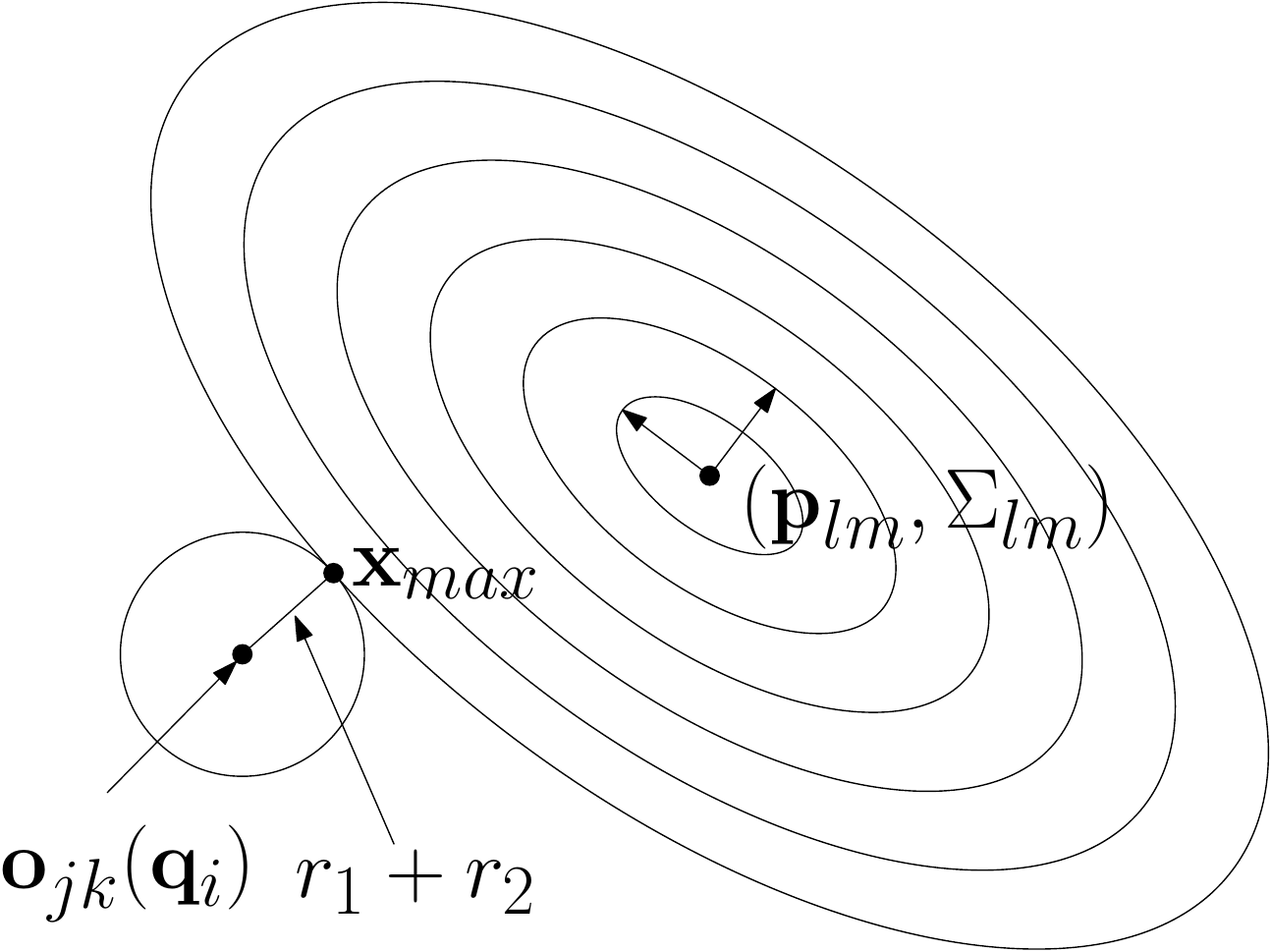}
  \caption{Approximation of probabilistic collision detection between a sphere obstacle of radius $r_2$ with a probability distribution $\mathcal{N} (\mathbf p_{lm}, \mathbf \Sigma_{lm})$ and a rigid sphere robot $B_{jk}(\mathbf q_i)$ centered at $\mathbf o_{jk}(\mathbf q_i)$ with radius $r_1$. It is computed as the product of the probability at $\mathbf x_{max}$ with the volume of the sphere with the radius computed as the sum of two radii, $V=\frac{4\pi}{3}(r_1+r_2)^3$.}
  \label{fig:prob_approx}
\end{figure}

Fig.~\ref{fig:prob_approx} illustrates how $ P_{col}(i,j,k,l,m)$ can be computed for $S_{lm}(t_i) \sim \mathcal{N} (\mathbf p_{lm}, \mathbf \Sigma_{lm})$.
If we assume that the robot's bounding volume $B_{jk}(\mathbf q_i)$ is a sphere centered at $\mathbf o_{jk}(\mathbf q_i)$, similar to the environment bounding volume $S_{lm}$, and denote the radii of $B_{jk}$ and $S_{lm}$ as $r_1$ and $r_2$, respectively, the exact probability of collision between them is given as:
\begin{equation}
\label{eq:colobj}
\begin{split}
P_{col}(i,j,k,l,m)=\int_{\mathbf x}I(\mathbf x,\mathbf o_{jk}(\mathbf q_i))p(\mathbf x,\mathbf p_{lm},\mathbf \Sigma_{lm})d \mathbf x,\\
\end{split}
\end{equation}
where the indicator function $I(\mathbf x,\mathbf o)$ and the obstacle function $p(\mathbf x,\mathbf p,\mathbf \Sigma)$ are defined as,
\begin{align}
\label{eq:colobj1}
I(\mathbf x,\mathbf o)=\left\{\begin{matrix}
1 & \textrm{if}\: \|\mathbf x - \mathbf o\| \leq (r_1+r_2) \\ 
0 & \textrm{otherwise}
\end{matrix}\right. \, \textrm{and}
\end{align}
\begin{align}
\label{eq:colobj2}
p(\mathbf x,\mathbf p,\mathbf \Sigma)=\frac{e^{-0.5(\mathbf x-\mathbf p)^T\mathbf \Sigma^{-1}(\mathbf x-\mathbf p)}}{\sqrt{(2\pi)^3\|\mathbf \Sigma\|}},
\end{align}
respectively.
It is known that there is no closed form solution for (\ref{eq:colobj}). 
Toit and Burdick approximate (\ref{eq:colobj}) as $V \cdot p(\mathbf o_{jk}(\mathbf q_i),\mathbf p_{lm},\mathbf \Sigma_{lm})$, where $V$ is the volume of sphere, i.e., $V=\frac{4\pi}{3}(r_1+r_2)^3$~\cite{du2011probabilistic}.
However, this approximated probability can be either smaller or larger than the exact probability. If the covariance $\mathbf \Sigma_{lm}$ is small, the approximated probability can be much smaller than the exact probability.
In order to compute an upper bound  on the collision probability, we compute $\mathbf x_{max}$, the position has the maximum probability of  $\mathcal{N} (\mathbf p_{lm}, \mathbf \Sigma_{lm})$ in $\mathbf B_{jk}(\mathbf q_i)$, and compute the upper bound of $P_{col}(i,j,k,l,m)$ as 
\begin{align}
\label{eq:approx}
P_{approx}(i,j,k,l,m) = V \cdot  p(\mathbf x_{max},\mathbf p_{lm},\mathbf \Sigma_{lm}).
\end{align}
Although $\mathbf x_{max}$ has no closed-form solution, it can be computed efficiently.
\begin{lemma}
\label{thm:lemmamax}
$\mathbf x_{max}$, the position has the maximum probability of  $\mathcal{N} (\mathbf p_{lm}, \mathbf \Sigma_{lm})$ in $\mathbf B_{jk}(\mathbf q_i)$, is formulated as an one-dimensional search of a parameter $\lambda$,
\begin{align}
\label{eq:lemma}
\mathbf x_{max}&=\left\{ \mathbf x|\|\mathbf x -\mathbf o_{jk}(\mathbf q_i)\|=(r_1+r_2) \,\textrm{and}\, \mathbf x \in \mathbf x(\lambda)\right\}, \textrm{where}\\
\mathbf x(\lambda)&=(\mathbf \Sigma_{lm}^{-1}+\lambda \mathbf I)^{-1}(\mathbf \Sigma_{lm}^{-1}\mathbf p_{lm}+\lambda \mathbf o_{jk}(\mathbf q_i)).
\end{align}
\end{lemma}
\begin{proof} 
The problem of finding the position with the maximum probability in a convex region can be formulated as an optimization problem with a Lagrange multiplier $\lambda$~\cite{groetsch1984theory},
\begin{align}
\label{eq:lemmaproof1}
\mathbf x_{max} = \argmin_{\mathbf x} \left\{ (\mathbf x-\mathbf p_{lm})^T \mathbf \Sigma_{lm}^{-1}(\mathbf x - \mathbf p_{lm})+\lambda (\mathbf x - \mathbf o_{jk})^2\right\}. 
\end{align}
The solution of (\ref{eq:lemmaproof1}) satisfies
\begin{align}
\label{eq:lemmaproof2}
&\triangledown  \left\{ (\mathbf x-\mathbf p_{lm})^T \mathbf \Sigma_{lm}^{-1}(\mathbf x - \mathbf p_{lm})+\lambda (\mathbf x - \mathbf o_{jk})^2\right\}=0,
\end{align}
and can be computed as
\begin{align}
\label{eq:lemmaproof3}
&2\mathbf \Sigma_{lm}^{-1}(\mathbf x - \mathbf p_{lm})+2\lambda(\mathbf x-\mathbf o_{jk})=0\\
&\mathbf x = (\mathbf \Sigma_{lm}^{-1}+\lambda \mathbf I)^{-1})(\mathbf \Sigma_{lm}^{-1}\mathbf p_{lm} + \lambda \mathbf o_{jk}).
\end{align}
\end{proof}

The approximated probability (\ref{eq:approx}) is guaranteed as an upper bound of the exact collision probability (\ref{eq:colobj}).
\begin{theorem}
\label{thm:approx}
The approximated probability $P_{approx}(i,j,k,l,m)$ (\ref{eq:approx}) is always greater or equal to the exact collision probability $P_{col}(i,j,k,l,m)$ (\ref{eq:colobj}).
\end{theorem}
\begin{proof} 
$p(\mathbf x_{max},\mathbf p_{lm},\mathbf \Sigma_{lm}) \ge p(\mathbf x,\mathbf p_{lm},\mathbf \Sigma_{lm})$ for $\{\mathbf x|\|\mathbf x - \mathbf o_{jk}(\mathbf q_i)\| \leq (r_1+r_2)\}$ from Lemma~\ref{thm:lemmamax}.
Therefore,
\begin{align}
\label{eq:theorem3}
P_{approx}(i,j,k,l,m)&=V \cdot p(\mathbf x_{max},\mathbf p_{lm},\mathbf \Sigma_{lm}) \\
&=  \int_{\mathbf x}I(\mathbf x,\mathbf o_{jk}(\mathbf q_i))d \mathbf x \cdot p(\mathbf x_{max},\mathbf p_{lm},\mathbf \Sigma_{lm})\\
&=\int_{\mathbf x}I(\mathbf x,\mathbf o_{jk}(\mathbf q_i))\cdot p(\mathbf x_{max},\mathbf p_{lm},\mathbf \Sigma_{lm})d \mathbf x \\
&\ge \int_{\mathbf x}I(\mathbf x,\mathbf o_{jk}(\mathbf q_i))\cdot p(\mathbf x,\mathbf p_{lm},\mathbf \Sigma_{lm})d \mathbf x \\
&=P_{col}(i,j,k,l,m).
\end{align}
\end{proof}

\subsection{Comparisons with Other Algorithms}

\begin{figure}[ht]
  \centering
  \subfloat[Case I]
  {
    \includegraphics[width=0.2\linewidth]{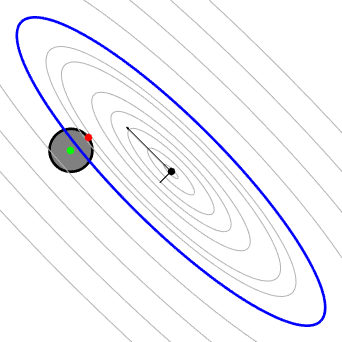}
  }
  \subfloat[Case II]
  {
    \includegraphics[width=0.2\linewidth]{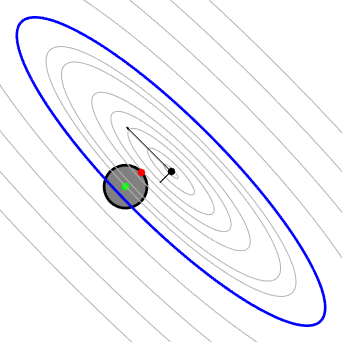}
  }
  \begin{tabular}{|c|p{1.5cm}|p{1.5cm}|}
  \hline
  \multirow{2}{*}{Algorithms}& \multicolumn{2}{|c|}{Collision probability}\\
  \cline{2-3}
  & \multicolumn{1}{c|}{Case I} & \multicolumn{1}{c|}{Case II}\\ \hline
  Numerical integration & \multicolumn{1}{r|}{0.09\%(O)} & \multicolumn{1}{r|}{1.72\%(X)}\\ \hline
  \begin{tabular}[x]{@{}c@{}}Enlarged bounding volumes\\  ($\delta_{CL} = 0.99$)~\cite{van2012lqg,Park:2012:ICAPS}\end{tabular}& \multicolumn{1}{r|}{100.00\%(X)} & \multicolumn{1}{r|}{ 100.00\%(X)} \\ \hline
  \begin{tabular}[x]{@{}c@{}}Approximation using\\ the center point PDF~\cite{du2011probabilistic}\end{tabular}  &  \multicolumn{1}{r|}{0.02\%(O)} & \multicolumn{1}{r|}{0.89\%(O)} \\ \hline
  Our approach &  \multicolumn{1}{r|}{0.80\%(O)} & \multicolumn{1}{r|}{8.47\%(X)} \\ \hline
  \end{tabular}

  \caption{{\bf Comparison of approximated collision probabilities for feasible (Case I) and infeasible (Case II) scenarios for $\delta_{CL}=0.99$:}
           We compare the exact collision probability (computed using numerical integration) with approximated probabilities of 1) enlarged bounding volumes (blue contour)~\cite{van2012lqg,Park:2012:ICAPS}, 2) approximation using object center point (in green)~\cite{du2011probabilistic}, and 3) our approach that uses the maximum probability point (in red). Our approach guarantees to not underestimate the probability, while the approximated probability is close to the exact probability. }
\label{fig:pcc_comparison}    
\end{figure}

In Fig.~\ref{fig:pcc_comparison}, we illustrate two cases of the collision probability computation between a circle $B$ (in gray), and a point (in black) $\mathbf x$ which has uncertainties, $\mathbf x \sim (\mathbf p, \mathbf \Sigma)$, in 2D.
We evaluate the exact collision probabilities using the numerical integration of the PDF. The collision probability of Case I is $0.09\%$, which is feasible with $\delta_{CL} = 0.99$, while the probability of Case II is $1.72\%$, which is infeasible.
Contours represent the bounds for different confidence levels, where the blue contour corresponds to $\delta_{CL}=0.99$.  
In both cases, the blue contour intersects with $B$ and approaches that use enlarged bounding volumes~\cite{van2012lqg,Park:2012:ICAPS} determine the objects are collide, while the collision probability for Case I is $0.09\%$.

Du Toit and Burdick~\cite{du2011probabilistic} used the probability of the center point (shown in green in Fig.~\ref{fig:pcc_comparison}) to compute a collision probability that is close to the actual value. However, their approach cannot guarantee upper bounds, and the approximated probability can significantly smaller from the actual probability if the covariance value is small. Case II in Fig.~\ref{fig:pcc_comparison} shows that the approximated probability is $0.89\%$, that satisfies the safety with the $\delta_{CL} = 0.99$, which is not true for  the exact probability $1.72\%$.

Unlike~\cite{du2011probabilistic}, we approximate the probability of the entire volume using the maximum probability value of a single point (shown in red in Fig.~\ref{fig:pcc_comparison}), as described in Section~\ref{subsec:pcc}. Our approach guarantees computation of the upper bound of collision probability, while the approximated probability is close to the exact probability than the enlarged bounding volume approaches.

\section{Belief State Estimation}
\label{sec:environment}

In this section, we describe our approach for computing the current state $\mathbf p$ of environment obstacles, and use that to estimate the current belief state $\mathbf b_t$ and future states $\mathbf b_i \: (i>t)$, which are represented as the probability distributions.
We construct or update the belief state of the environment $\mathbf b = (\mathbf p, \mathbf \Sigma)$ using means and covariances $\mathbf p_{ij}$ and $\mathbf\Sigma_{ij}$ of the poses of the existing bounding volumes $S_{ij}$. That is, $\mathbf p = \begin{bmatrix} \mathbf p_{11}^T & ... & \mathbf p_{lm}^T \end{bmatrix} ^T$ and $\mathbf\Sigma = \textrm{diag}(\mathbf\Sigma_{11},...,\mathbf\Sigma_{lm})$, where $\mathbf \Sigma$ is a block diagonal matrix of the covariances. 

\subsection{Environment State Model}
\label{subsec:env_state}

In order to compute reliable obstacle motion trajectories in dynamic environments, first  it is important to gather the state of obstacles using sensors. There is considerable work on pose recognition in humans~\cite{plagemann2010real,shotton2013real} or non-human objects~\cite{lepetit2005randomized} in computer vision and related areas.

\begin{figure}[ht]
  \centering
  \includegraphics[width=0.5\textwidth]{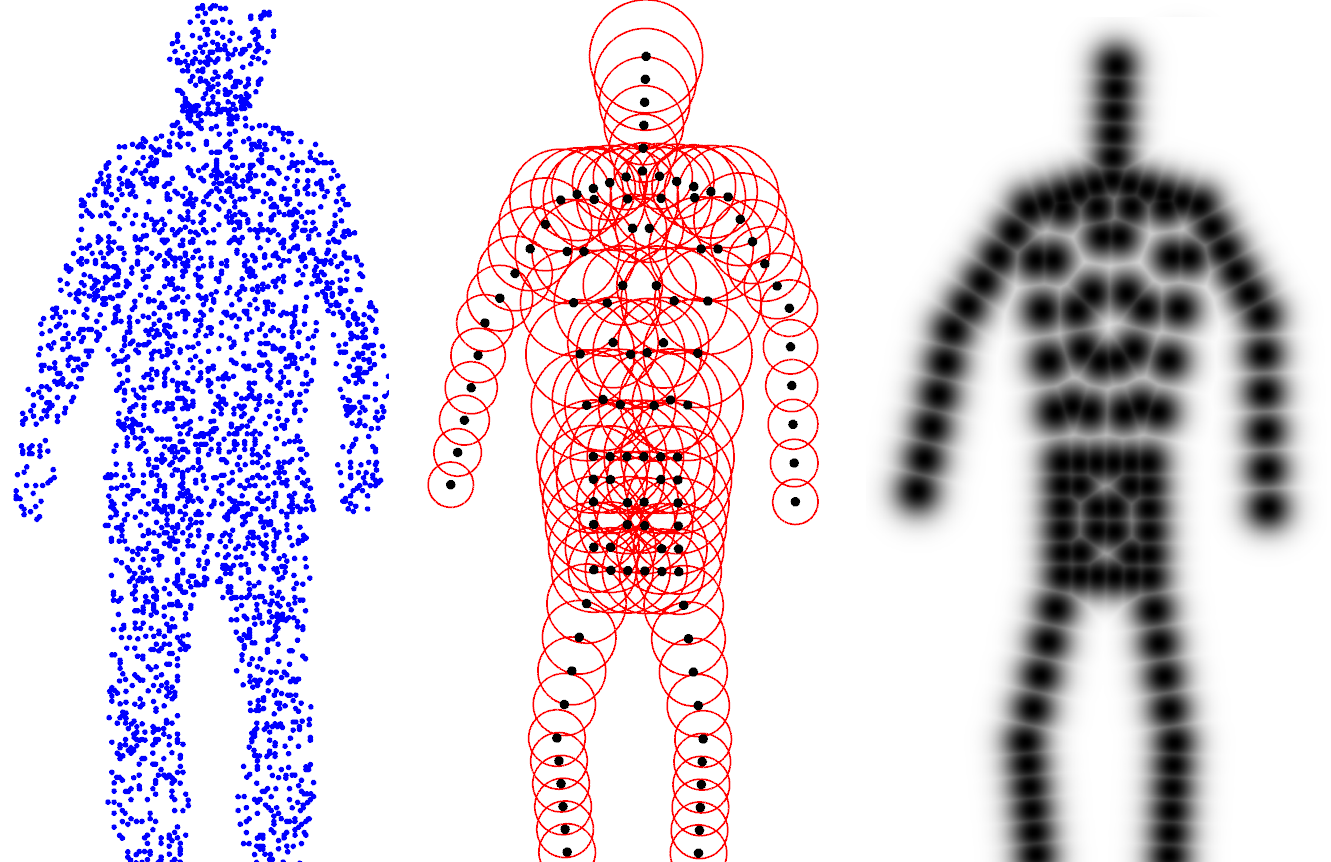}
  \caption{{\bf Environment belief state estimation for a human obstacle:} We approximate the point cloud from the sensor data using bounding volumes. The shape of bounding volumes are pre-known in the database, and belief states are defined on the probability distributions of bounding volume poses:
  (a) input point clouds (blue dots) (b) the bounding volumes (red spheres)with their mean positions (black dots) (c) the probabilistic distribution of mean positions. 0\% confidence level (black) to 100\% confidence level (white).}
  \label{fig:ho}
\end{figure}

We assume that a model database is given that consists of pre-defined shape models for each moving obstacle in the environment; e.g., an obstacle may correspond to a known shape such as a ball or a human arm.  Furthermore, we are also given a bounding volume approximation of each such model.  In particular, we use spheres as the underlying bounding volumes (Fig.~\ref{fig:ho}), as they provide an efficient approximation for computing the collision probability (see Section~\ref{subsec:pcc}). 

We segment out the background pixels correspond to the known static environments from the depth map, and generate a point cloud which is used to compute the best approximating environment state $\mathbf{p}^*$.
It can be computationally inefficient to estimate and predict the states of dynamic obstacles that are represented using a large number of point clouds. Therefore, we use a reduced environment state representation that is defined in terms of the positions and velocities of the dynamic obstacles and utilize the predefined shape models for the dynamic obstacles. Each shape model for an obstacle in the model database is defined with multiple bounding volume shapes and their initial poses.
For the input point cloud, we perform the object recognization at the beginning frame, then optimize $\mathbf{p}^*$ using the Ray-Constrained Iterative Closest Point~\cite{ganapathi12realtime} algorithm.

Given the predefined shape model for each obstacle, ICP algorithm computes the best approximating environment state $\mathbf{p}^*$ for the input point clouds $\mathbf d_1,..., \mathbf d_n$.
The likelihood of $\mathbf d_k$ for an environment state $\mathbf p$ is modeled as
\begin{equation}
P_{pc}(\mathbf d_k | \mathbf{p}) \propto \exp \left( {-} \frac{1}{2} \min_{i,j} \|S_{ij}-\mathbf d_k\| ^2 \right),
\label{eq:probability_model}
\end{equation}
and the optimal environment state $\mathbf{p}^*$ that maximizes the likelihood of the each point cloud is computed with two additional constraints, represented as $C_1$ and $C_2$:
\begin{align}
\begin{split}
&\mathbf p^* = \argmax_{\mathbf{p}} = \prod_k P_{pc}(\mathbf d_k | \mathbf{p}),  \\
\text{subject to} \:  C_1:& \forall (\mathbf p_{ij}, \mathbf p_{ik}) : (1 - \epsilon) \leq  \frac{||\mathbf{p}_{ij} - \mathbf{p}_{ih}||}{c_{dist}({{ij}, {ih})}}  \leq (1 + \epsilon) \\
C_2:& \forall \mathbf S_{ij} \forall \mathbf s_i  : \textrm{proj}_{\mathbf s_i} (\mathbf S_{ij}) \subset \textrm{proj}_{\mathbf s_i}(\mathbf d_1,...,\mathbf d_n),
\end{split}
\label{eq:maximization},
\end{align}
where ${c_{dist}({ij}, {ih})}$ is the distance between $\mathbf p_{ij}$ and $\mathbf p_{ih}$ of the predefined shape model, and $\textrm{proj}(\mathbf s_i)$ represents a projection to the 2D image space of depth sensor $\mathbf s_i$. Constraint $C_1$ corresponds to the length preserving constraint for the bounding volumes belong to the same object. $C_2$ ensures that the correct point clouds are generated for $\mathbf S_{ij}$ in view of all sensors $\mathbf s_i$.

\subsection{Belief State Estimation and Prediction}
\label{subsec:env_belief}

The optimal solution $\mathbf p^*$ computed in Section~\ref{subsec:env_state} can have erros due to the sensors (e.g., point-cloud sensors) or poor sampling. 
Furthermore, obstacle motion can be sudden or abrupt and this can result in various uncertainties in the prediction of future motion.

At each time $t$, we use the Kalman filter to estimate the position and velocity of the bounding volume $\mathbf S_{ij}$.
We estimate the current belief states $\mathbf b_t = (\mathbf p_t, \Sigma_t)$ from the history of observed environment states $\mathbf{p}^*$, and then also predict the future state of the environment that is used for probabilistic collision checking.
Its state at time $t$ is represented as
\begin{align}
(\mathbf{x}_{ij})_t = \begin{bmatrix} (\mathbf p_{ij})_t^T & (\mathbf{\dot{p}}_{ij})_t^T\end{bmatrix}^T,
\end{align}
where $(\mathbf p_{ij})_t$ is the position of $\mathbf S_{ij}$ at time $t$. We will omit subscript $_{ij}$ when we refer to a single obstacle.
Using the Kalman filter, we estimate $\mathbf{x}_t$ as
\begin{align}
\mathbf{x}_t &= \mathbf{A} \mathbf{x}_{t-1} + \mathbf{B} \mathbf{u}_t + \mathbf{w}_t, \label{eq:KF_predict} \\
\mathbf{z}_t &= \mathbf{C} \mathbf{x}_t + \mathbf v_t, \label{eq:KF_udpate}
\end{align}
where the matrices are defined  as
\begin{equation}
\mathbf A = \begin{bmatrix}
I _{3\times3} & \Delta tI _{3\times3}\\ 
0 & I _{3\times3}
\end{bmatrix},
\mathbf{B} = 
\begin{bmatrix}
I _{3\times3} \\ 
\Delta tI _{3\times3}
\end{bmatrix},
\mathbf{C} = 
\begin{bmatrix}
I _{3\times3} & 0
\end{bmatrix},
\end{equation}
and $\mathbf w_t$ and $\mathbf v_t$ are the process noise and observation noise, respectively.
$\mathbf{z}_t$ is an observation that corresponds to $\mathbf p^*$.

Although we cannot directly control the environment, we compute an hypothetical input $\mathbf{u}_t$ that is used to preserve the distances between the bounding volumes belong to the same object in the predicted result. 
During the estimation, if the distance of an object $\mathbf S_{ij}$ from another object $\mathbf S_{ih}$ exceeds the distance in the predefined shape model, we compute an appropriate value for $\mathbf{u}_t$ to preserve the initial distance.
In order to preserve the initial distance $\|(\mathbf p_{ij})_0-(\mathbf p_{ih})_0\|$, we pull the $\mathbf S_{ij}$'s position $(\mathbf p_{ij})_t$ toward $\mathbf S_{ih}$'s position $(\mathbf p_{ih})_t$ using
\begin{equation}
\label{eq:length_control}
\mathbf{u}_{t} = \left( (\mathbf{p}_{ih})_t - (\mathbf{p}_{ij})_t \right) \left( 1 - \frac{\|(\mathbf p_{ij})_0-(\mathbf p_{ih})_0\|}{\|(\mathbf p_{ij})_t-(\mathbf p_{ih})_t\|} \right) .
\end{equation}

\subsection{Spatial and Temporal Uncertainties in Belief State}
\label{subsec:uncertainties}

\begin{figure}[ht]
  \centering
  \subfloat[]
  {
    \includegraphics[width=0.22\linewidth]{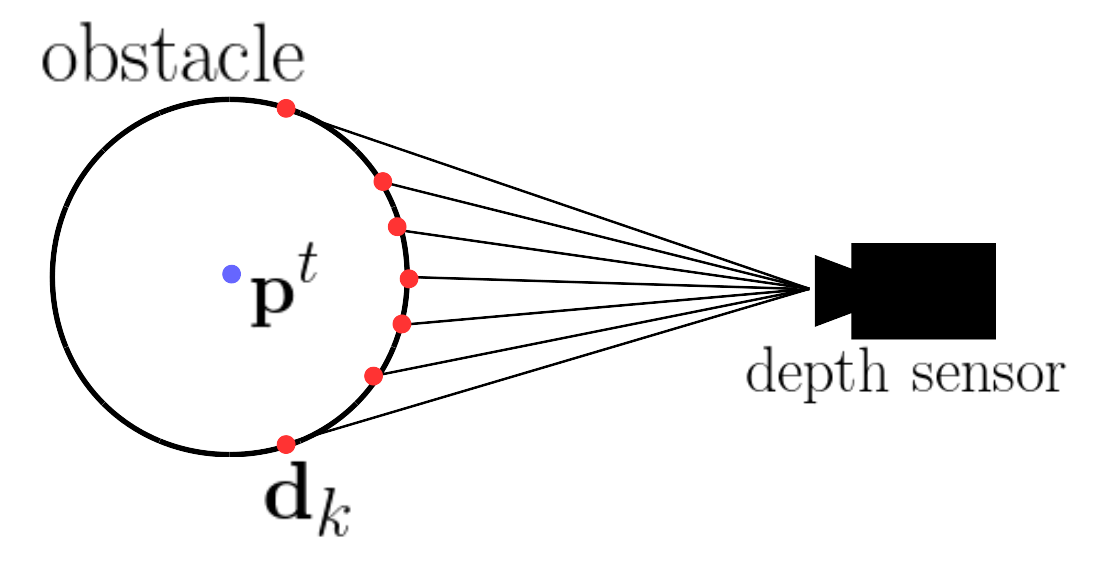}
  }
  \subfloat[]
  {    
    \includegraphics[width=0.22\linewidth]{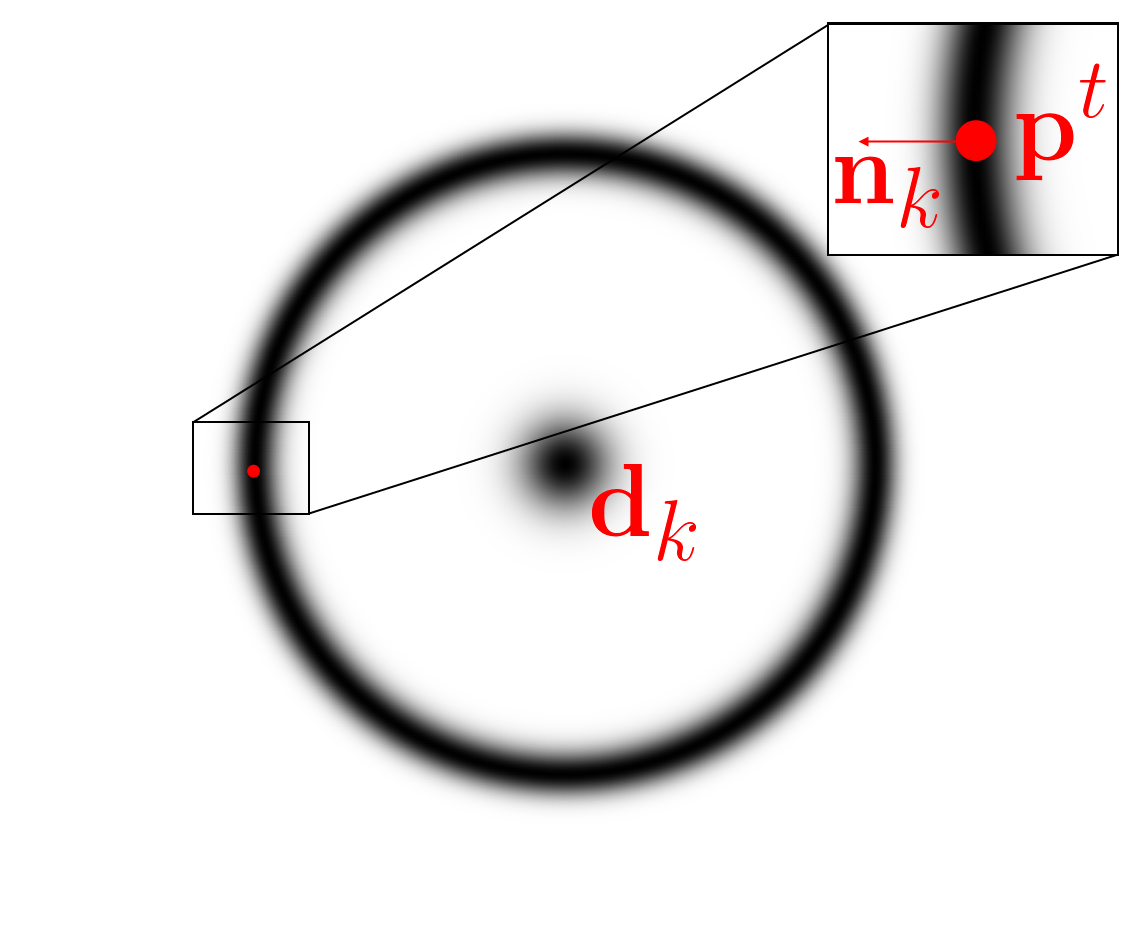}
    \label{fig:sensor_error_single}
  }
  \subfloat[]
  {
    \includegraphics[width=0.22\linewidth]{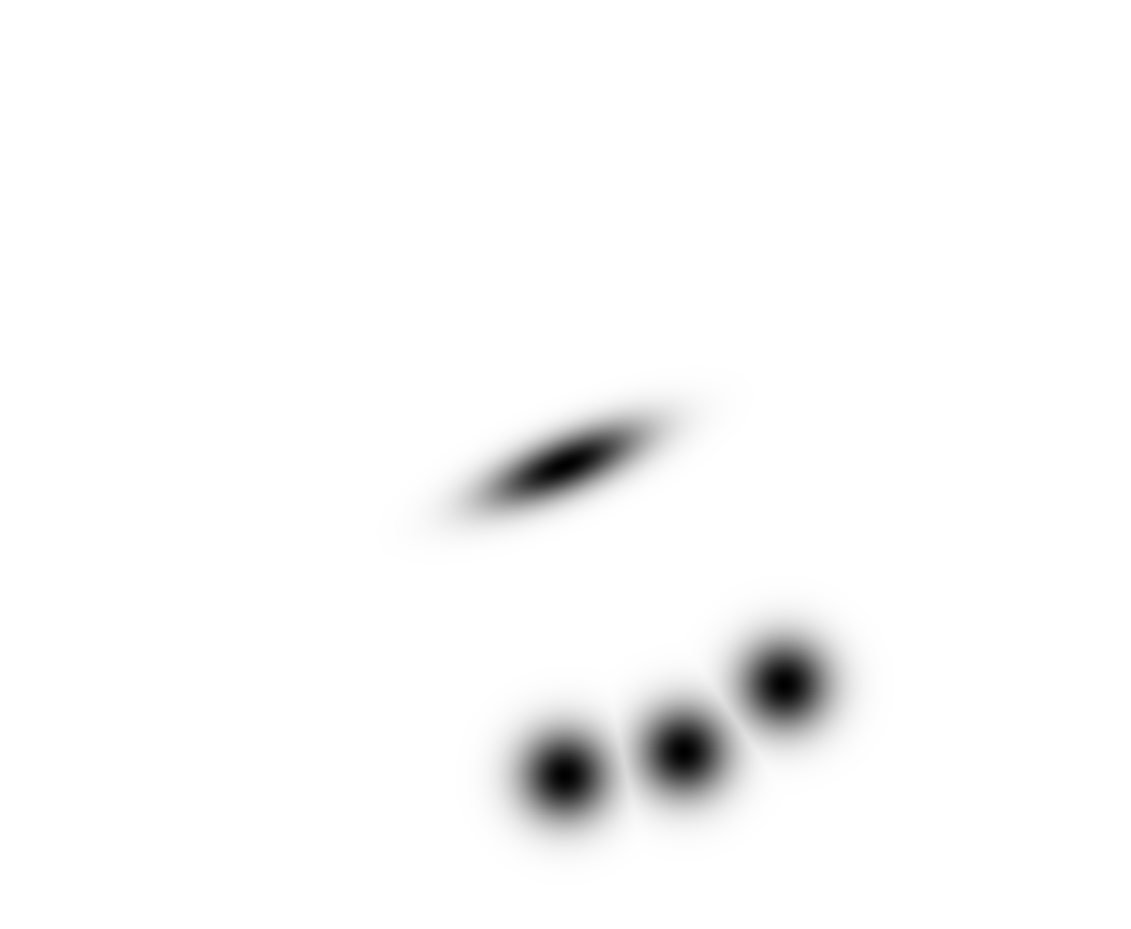}
  }
  \subfloat[]
  {
    \includegraphics[width=0.22\linewidth]{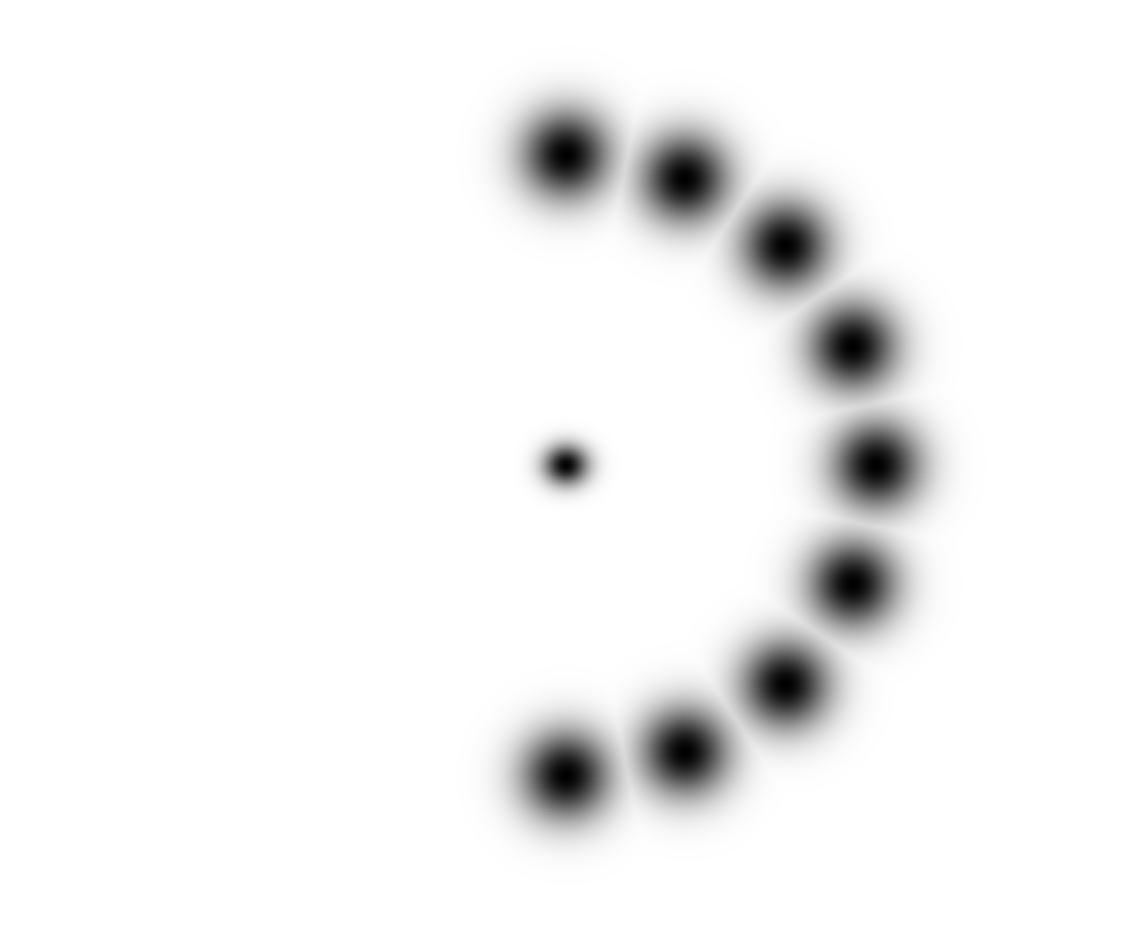}
  }
  \caption{{\bf Spatial uncertainty:} 
           (a) Sphere obstacle and its point cloud samples from a depth sensor.
           (b) Probability distribution of a sphere center state $\mathbf p$ for a single point cloud $\mathbf d_k$.
           (c) Probability distribution of $\mathbf p$ for a partially visible obstacle.
           (d) Probability distribution of $\mathbf p$ for a fully visible obstacle.
           }
\label{fig:sensor_error}           
\end{figure}
During the environment state estimation, spatial uncertainty or errors arise from the  resolution of the sensor. It is known that the depth sensor error can be modeled as Gaussian distributions around each point $\mathbf d_k$~\cite{nguyen2012modeling}. We assume that the center of distribution is $\mathbf d_k$ itself and the covariance is isotropic and can be represented as $\sigma_{s}^2 I_{3\times3}$. Due to the sensor error, the optimal environment state $\mathbf p^*$ computed from (\ref{eq:maximization}) may differ from the true environment state $\mathbf p^{t}$.

We derive the equation for the observation noise $\mathbf v_t$ in (\ref{eq:KF_udpate}) for an environment state computed using (\ref{eq:maximization}). For simplicity, we assume the environment has only one sphere with radius $r$ and its optimal state is computed from point clouds (Fig.~\ref{fig:sensor_error}(a)). For a single obstacle case, the optimization equation (\ref{eq:maximization}) can be written as
\begin{align}
P(\mathbf p) \propto \max_{\mathbf{p}} \quad & \prod_k \exp \left( {-} \frac{1}{2} \left( ||\mathbf{p} - \mathbf{d}_k|| - r \right)^2 \right) \notag \\
= &\prod_k P( \mathbf{p} | \mathbf{d}_k ). \label{eq:maximization2}
\end{align}
Here, $P( \mathbf{p} | \mathbf{d}_k )$ corresponds to the spherical probability distribution that represents the highest value at distance $r$.  
If $r \gg \sigma_{s}$, it can be approximated near $\mathbf{p}^{t}$ as a Gaussian distribution as shown in Fig.~\ref{fig:sensor_error}(b),
\begin{equation}
\label{eq:noise_approx}
P( \mathbf{p} | \mathbf{d}_k ) \sim \mathcal{N}( \mathbf{p}^{t}, \sigma_{s}^2 \mathbf{n}^{t} \times (\mathbf{n}^{t}) ^T ),
\end{equation}
where $\mathbf{n}_k = (\mathbf{p}^t - \mathbf{d}_k) / ||\mathbf{p}^t - \mathbf{d}_k||$.

$P(\mathbf p)$ is a product of these spherical probability distributions (\ref{eq:noise_approx}) for different point cloud $\mathbf d_k$, and it corresponds to another Gaussian distribution $\mathcal{N}( \mathbf{p}_t, \Sigma^* )$. 
Therefore, the observation error $\mathbf{v}_t$ can be represented as:
\begin{align}
\mathbf{v}_t \sim P(\mathbf p) - \mathbf p^t = \mathcal{N}( \mathbf{0}, \Sigma^* ).
\end{align}
If we are given more samples from the sensor and there is less sensor error, the  error distribution becomes more centralized.

Temporal uncertainty arises due to discretization of the time domain, which corresponds to approximating the velocity of dynamic obstacle using forward differencing method. 
Let $\mathbf{x}(t)$ be the obstacle position at time $t$. By Taylor expansion, we obtain
\begin{equation}
\mathbf{x}(t + \Delta t) = \mathbf{x}(t) + \dot{\mathbf{x}}(t) \Delta t + \frac{1}{2} \ddot{\mathbf{x}}(t) \Delta t^2 + O(\Delta t^3),
\end{equation}
and 
\begin{equation}
\dot{\mathbf{x}}(t) \approx \frac{\mathbf{x}(t + \Delta t) - \mathbf{x}(t)}{\Delta t} + \frac{1}{2} \ddot{\mathbf{x}}(t) \Delta t + O(\Delta t^2). \label{eq:taylor_expansion}
\end{equation}
From the history of past environment states, we compute $\ddot{\mathbf{x}}(t)$ of each object and its covariance $\Sigma_a (t)$.
Based on Equation (\ref{eq:taylor_expansion}), we get  the process error $\mathbf w_t$ as
\begin{equation}
\mathbf w_t \sim \mathcal{N} \left( \mathbf{0},
\left[ \begin{array}{c|c} \frac{1}{4} (\Delta t)^4 \Sigma_a(t) & \mathbf{0} \\ \hline \mathbf{0} & \frac{1}{4} (\Delta t)^2 \Sigma_a(t) \end{array} \right]
\right),
\end{equation}
which is used in our estimation framework (Section~\ref{subsec:env_belief}) to compute the environment belief states.
These estimated belief states are used for collision probability computation (Section~\ref{subsec:pcc}).

\section{Space-Time Trajectory Optimization}
\label{sec:optimization}

In this section, we describe our motion planning algorithm based on probabilistic collision detection (Section~\ref{sec:pcc}) and environment belief state estimation (Section~\ref{sec:environment}).

Fig.~\ref{fig:architecture} highlights various components of our planning algorithm. The pseudo-code description is given in Algorithm 1 for a single planning step $\delta T$.

\begin{figure}[t]
  \centering
  \includegraphics[trim=0in 0in 0in 0in, clip=true, width=0.9\linewidth]{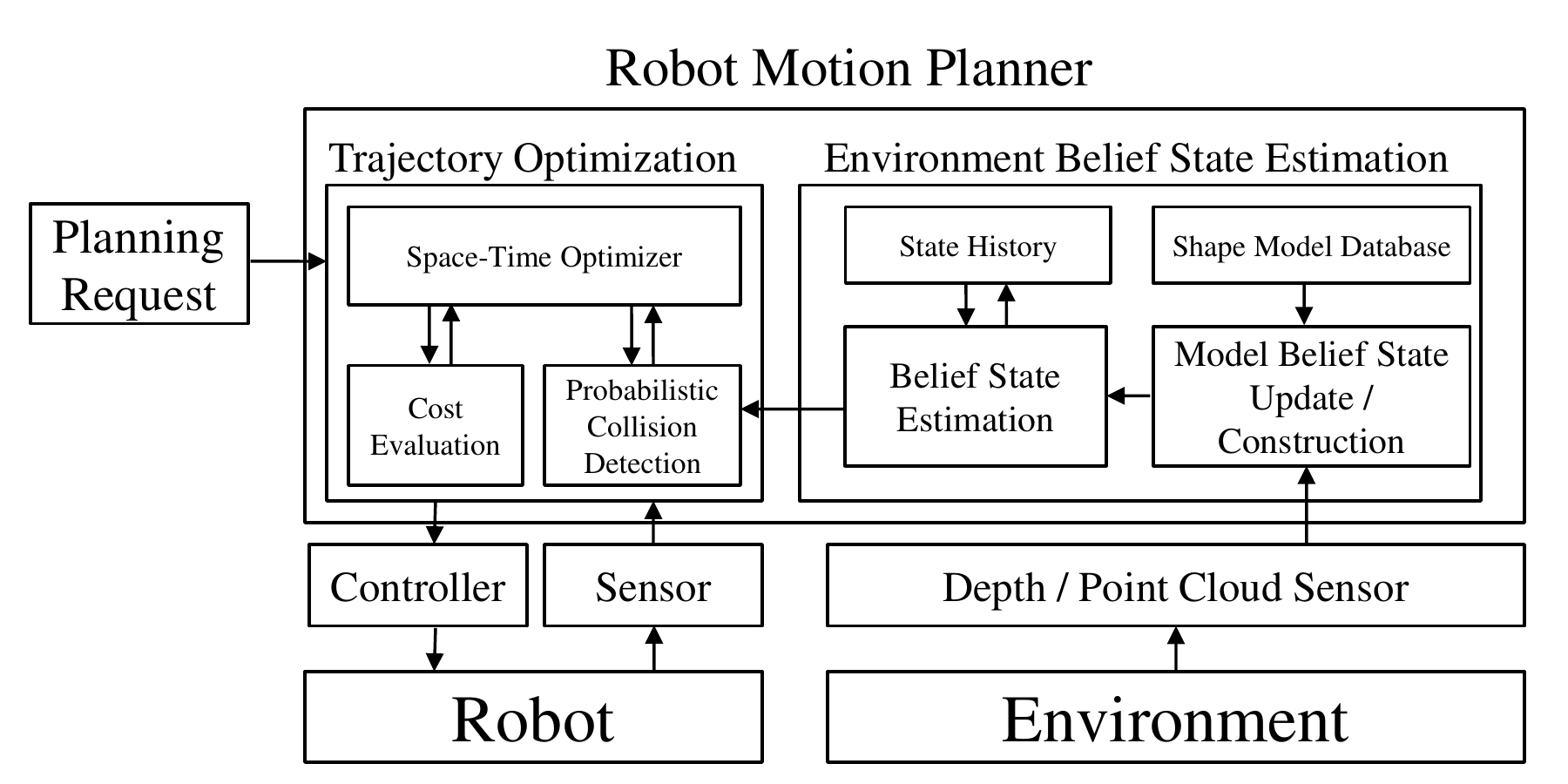}
  \caption{{\bf Trajectory Planning:} We highlight various components of our algorithm. These include belief space estimation from the sensor data and environment description, probabilistic collision checking, and trajectory optimization.}
  \label{fig:architecture}
\end{figure}

\begin{algorithm}[t]
 \caption{$\mathbf Q^*=$PlanWithEnvUncertainty($\mathbf Q$, $\{\mathbf d_k\}, t_i$) \\: Compute the optimal robot trajectory $\mathbf Q^*$ during the planning step $\Delta T$ for the environment point clouds $\{\mathbf d_k\}$ at time $t_i$}
  \label{alg:pseudo}
  \begin{algorithmic}[1]
    \REQUIRE initial trajectory $\mathbf Q$, environment point clouds $\{\mathbf d\}$, time $t_i$
    \ENSURE Optimal robot trajectory $\mathbf Q^*$ for time step $\Delta T$
    \STATE $\mathbf p_i$ = EnvironmentStateComputation($\{\mathbf d_k\}$) // {\em compute the environment state of dynamic obstacles}
    \FOR {$k \in \{i, ..., i + \Delta T\}$}
    \STATE $\mathbf B_k$ = BeliefStateEstimation($\mathbf B_0, ..., \mathbf B_{k-1}$, $\mathbf p_i$) //{\em estimate the current and future belief states}
    \ENDFOR
    \WHILE {elapsed time $<\Delta T$}
    \STATE $P$=ProbCollisionChecking($\mathbf Q,\{\mathbf B_i,...,\mathbf B_{i+\Delta T}\}$) // {\em perform probabilistic collision detection}
    \STATE $\mathbf Q^*$=Optimize($\mathbf Q,P$) // {\em compute the optimal trajectory for high-DOF robot}
    \ENDWHILE
  \end{algorithmic}
\end{algorithm}

As described in Section.~\ref{sec:environment}, we construct or update the belief state of the environment $\mathbf b = (\mathbf p, \Sigma)$, which is the probability distribution of the poses of the existing bounding volumes. Then we predict the future belief state of the environment.

We define the time-space domain $\mathcal X$, which adds a time dimension to the configuration space, i.e., $\mathcal X = \mathcal C \times T$.
The robot's trajectory, $\mathbf q(t)$, is represented as a function of time from the start configuration $\mathbf q_s$ to the goal configuration $\mathbf q_g$. It is represented using the matrix $\mathbf Q$,
\begin{equation}
\label{eq:x}
\mathbf Q=\begin{bmatrix} \mathbf q_s &  \mathbf q_1&...&  \mathbf q_{n-1}  &   \mathbf q_g\\t_0 & t_1&...&t_{n-1}&t_n\end{bmatrix},
\end{equation}
 which corresponds to $n+1$ configurations at discretized keyframes,  $t_i=i \Delta_T$, which have a fixed interval $\Delta_T$. We denote the $i$-th column of $\mathbf Q$ as $\mathbf x_i=\begin{bmatrix}\mathbf q_i^T & t_i \end{bmatrix}^T$.

Given the initial and goal positions for motion planning, our planner computes a locally optimal trajectory based on the objective function defined for the duration of the trajectory and also based on other constraints (e.g., smoothness). We use incremental trajectory optimization, which repeatedly refines a motion trajectory using an optimization formulation~\cite{Park:2012:ICAPS}.
The planner initializes the robot trajectory $\mathbf Q$ as a smooth trajectory of predefined length $T$ between the start and goal configurations $\mathbf q_s$ and $\mathbf q_g$, i.e.,
\begin{align}
\label{eq:traj_init}
\mathbf Q=\argmin_{Q}\sum_{i=1}^{n-1}\|\mathbf q_{i-1}-2\mathbf q_i+\mathbf q_{i+1}\|^2.
\end{align}
The trajectory is refined during every planning step $\Delta T$ based on various constraints, and we add collision probability constraints which is based on the probabilistic collision described in Section~\ref{sec:pcc} as the collision-free constraints.

We define the collision probability constraint of feasible robot trajectories based on the following probability computation formulation (shown as $P()$): 
\begin{equation}
\label{eq:pcol}
\begin{split}
\forall \mathbf x_i  : P(\mathbf q_i \in \mathcal{C}_{obs}(t_i)) < 1 - \delta_{CL}.
\end{split}
\end{equation}
The computed trajectories that satisfy (\ref{eq:pcol}) guarantee that the probability of collision with the obstacles is bounded by the confidence level $\delta_{CL}$, i.e. the probability that a computed trajectory has no collision is higher than $\delta_{CL}$. Use of a higher confidence level computes safer, but more conservative trajectories. The use of a lower confidence level increases the success rate of planning, but also increases the probability of collision. 

The objective function for trajectory optimization at time $t_k$ can be expressed as the sum of trajectory smoothness cost, and collision constraint costs for dynamic uncertain obstacles and static known obstacles,
\begin{small}
\begin{equation}
\label{eq:opt}
\begin{split}
f(\mathbf Q)=\min_{Q}&\sum_{i=k+m}^{n}\left(\|\mathbf q_{i-1}-2\mathbf q_i+\mathbf q_{i+1}\|^2+C_{static}(\mathbf Q_i) \right)\\
+&\sum_{i=k+m}^{k+2m} \textrm{max}(P(\mathbf q_i \in \mathcal C_{obs}(\mathbf x_i))-(1-\delta_{CL}),0),
\end{split}
\end{equation}
\end{small}
where $m$ is the number of time steps in a planning time step $\Delta T$. Furthermore, we can add additional kinematic or dynamic constraints that the robot has to satisfy, such as bounds on the joint position, velocity limits or robot balancing constraints. These can be satisfied in the trajectory optimization framework by formulating them as a constraint optimization problem, with these specific constraints.

Unlike the previous optimization-based planning approaches~\cite{Park:2012:ICAPS,Zucker:IJRR:2012} which maintain and cannot change the predefined trajectory duration for the computed trajectory, our planner can adjust the duration of trajectory $T$ to avoid collisions with the dynamic obstacles.
When the trajectory planning starts from $\mathbf t_i$ ($\mathbf t_i$ can be different from $\mathbf t_s$ due to replanning) and if the computed trajectory $\mathbf Q$ violates the collision probability constraint (\ref{eq:pcol}) at time $j$, i.e., $P(\mathbf q_j \in \mathcal C_{obs}(t_j)) \ge \delta_{CL}$, we repeatedly add a new time step $\mathbf x_{new}$ before $\mathbf x_{j}$ and rescale the trajectory from $\left[\mathbf t_i,...,\mathbf t_{j-1}\right]$ to $\left[\mathbf t_i,...,\mathbf t_{j-1}, \mathbf t_{new}\right]$, until $\mathbf x_{new}$ is collision-free. Moreover, the next planning step starts from $\mathbf x_{new}$.
This formulation of adjusting the trajectory duration allows the planner to slow the robot down when it cannot find a safe trajectory for the previous trajectory duration due to the dynamic obstacles. 
The optimization problem in (\ref{eq:opt}) is solved using Covariant Hamiltonian trajectory optimization~\cite{Zucker:IJRR:2012}, which preserves the trajectory smoothness during  optimization.

If the optimization algorithm converges, our algorithm computes the  optimal trajectory,
\begin{align}
\mathbf Q^*=\argmin_{\mathbf Q}f(\mathbf Q),
\end{align}
which provides a probabilistic collision guarantee with the given confidence level $\delta_{CL}$, for the time step $\Delta T$.

\section{Results}
\label{sec:result}
In this section, we describe our implementation and highlight the performance of our planning algorithm on different benchmark scenarios.

We tested our planning algorithm in simulated environments with models of a 6-DOF Universal Robot UR5 (Fig.~\ref{fig:experiment1}(a)(b)) and a 7-DOF KUKA IIWA robot arm (Fig.~\ref{fig:experiment1}(c)(d)). The environments have some complex static obstacles such as tools or furniture in a room.
The dynamic obstacle is a human, and we assume that the robot operates in close proximity to the human, however, the human does not intend to interact with the robot. 
We use a Kinect as the depth sensor, which can represent a human as 25-30k point clouds. We use a commodity PC for the planner, and use OpenMP to compute the probabilistic collision checking in parallel using multi-core CPUs.

\subsection{Experimental Results}

\begin{table*}[ht]
\centering
\resizebox{\linewidth}{!}{
\begin{tabular}{|c|c|c|c|c|c|c|}
\hline
Benchmark & Robot DOF & \begin{tabular}[x]{@{}c@{}}\# of Robot\\Bounding\\ Volumes\end{tabular}  & \begin{tabular}[x]{@{}c@{}}\# of Samples\\ in Point Cloud\end{tabular} & \begin{tabular}[x]{@{}c@{}}Environment \\ State DOF\end{tabular} & \begin{tabular}[x]{@{}c@{}}Confidence \\ Level\end{tabular}  & \begin{tabular}[x]{@{}c@{}}Average \\ Planning\\ Time\end{tabular}  \\ \hline
Prediction 1 & 6 (UR5) & 30 & 33,000 & 336 & $0.95$ & 138.83 ms \\ \hline
Prediction 2 & 6 (UR5) & 30 & 29,500 & 336 & $0.95$ & 115.55 ms \\ \hline
\begin{tabular}[x]{@{}c@{}}Time-Space\\Search 1\end{tabular} & 7 (IIWA) & 35 & 35,000 & 336 & $0.95$ & 771.44 ms \\ \hline
\begin{tabular}[x]{@{}c@{}}Time-Space\\Search 2\end{tabular} & 7 (IIWA) & 35 & 35,500 & 336 & $0.95$ & 552.64 ms \\ \hline
\begin{tabular}[x]{@{}c@{}}Exact\\Collision\\Checking\end{tabular} & 7 (IIWA) & 35 & 35,000 & 336 & $1.0$ & 154.99 ms \\ \hline
Sensing Noise 1 & 7 (IIWA) & 35 & 35,000 & 336 & $0.95$ & 720.52 ms \\ \hline
Sensing Noise 2 & 7 (IIWA) & 35 & 35,000 & 336 & $0.99$ & 846.11 ms \\ \hline
\end{tabular}
}
\caption{{\bf Complexity and planning results in our benchmarks:} We use two different robot models UR5 and IIWA, in our benchmarks. We highlight the complexity of each benchmark in terms of robot bounding volumes, the number of point cloud samples, DOF of the environment state, and the confidence level used for probabilistic collision detection. We compute the average planning time for each benchmark on a multi-core CPU.}
\label{table:performance}
\end{table*}

\begin{figure}[t]
  \centering
  \subfloat
  {
    \includegraphics[width=0.22\linewidth]{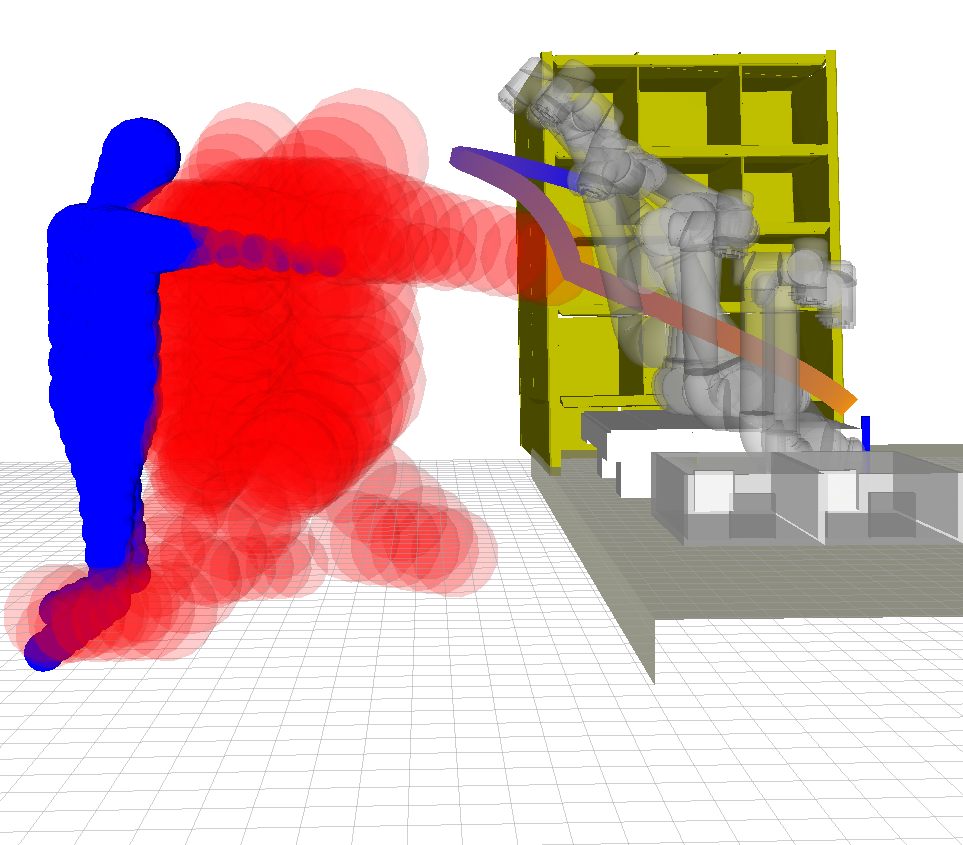}
  }
  \subfloat
  {
    \includegraphics[width=0.22\linewidth]{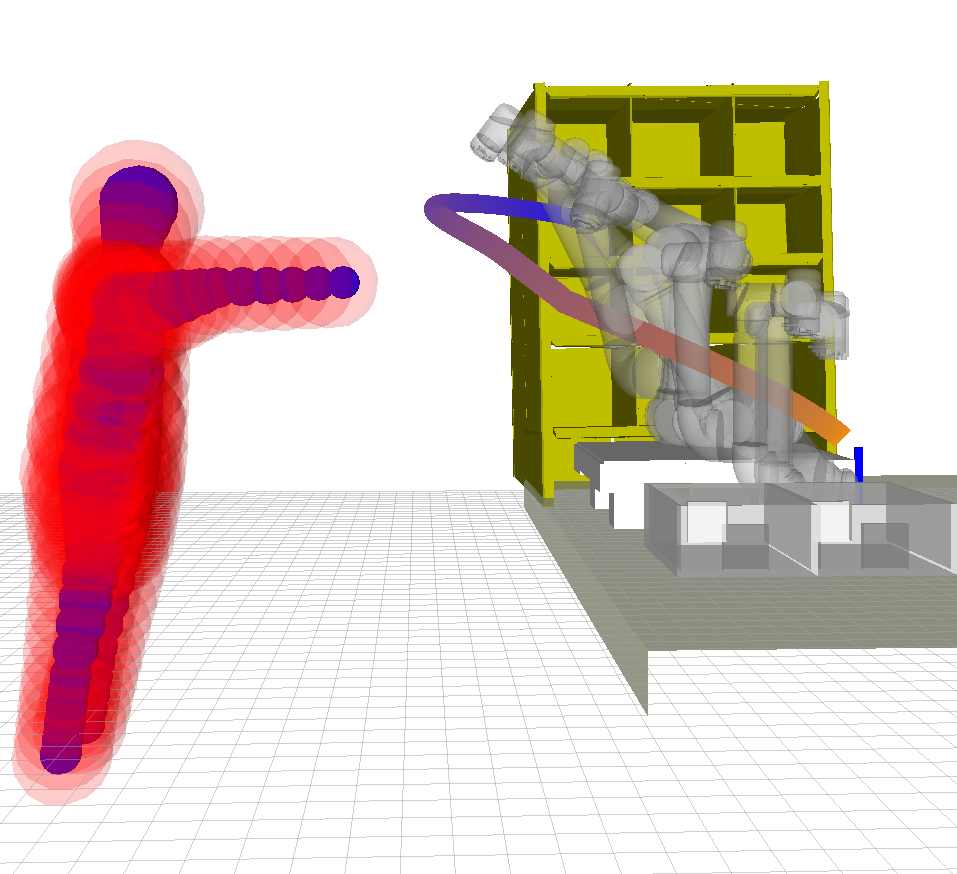}
  }
  \subfloat
  {
    \includegraphics[width=0.22\linewidth]{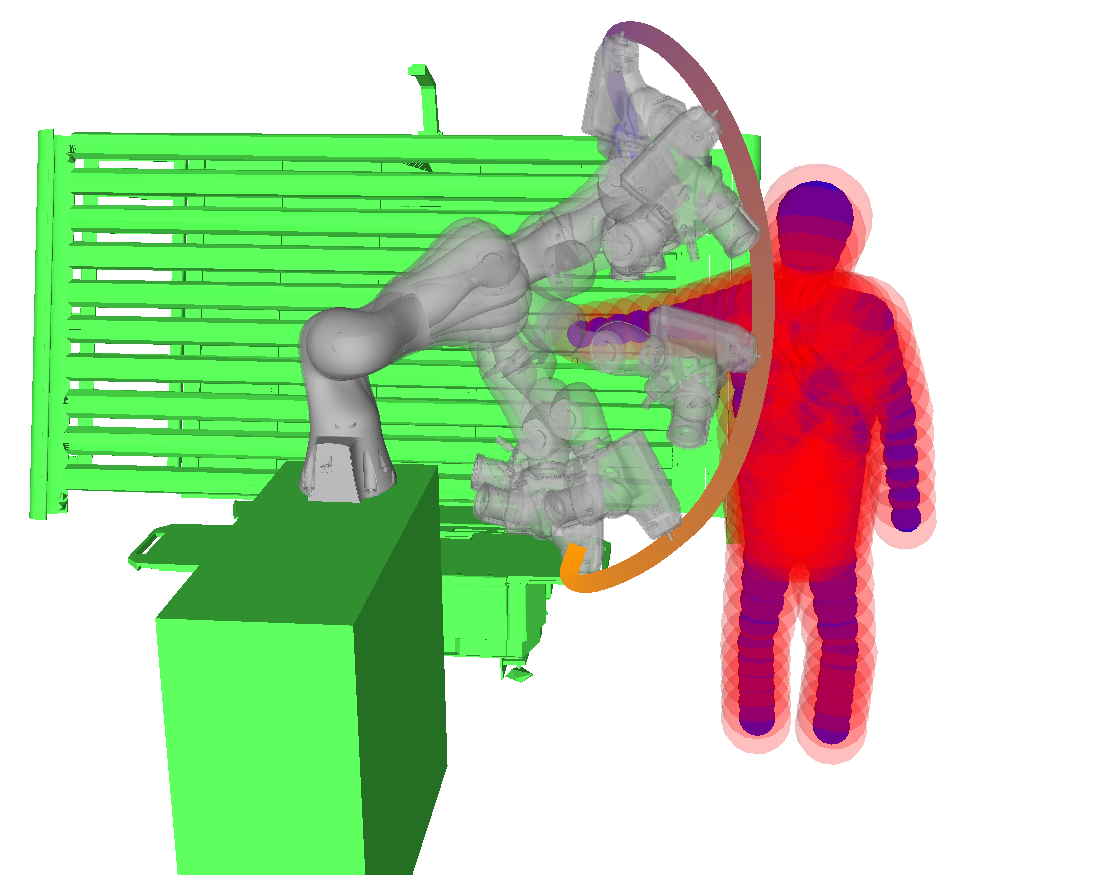}
  }
  \subfloat
  {
    \includegraphics[width=0.22\linewidth]{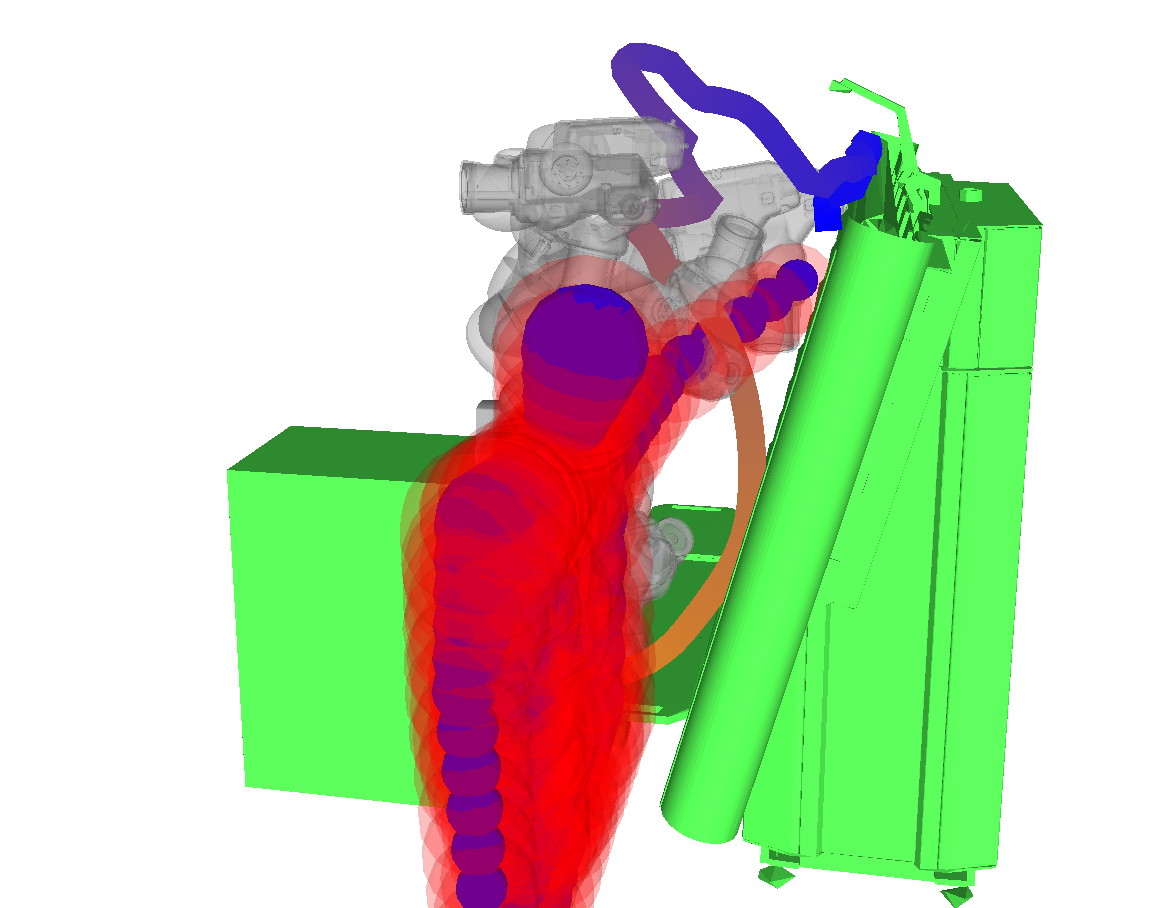}
  }
  \caption{{\bf Robot Trajectory with Dynamic Human Obstacles:} Static obstacles are shown in green, the estimated current and future human bounding volumes are shown in blue and red, respectively.
           (a) When a human is approaching the 6-DOF robot arm (UR5), our planner changes its trajectory to avoid potential future collisions.
           (b) When a standing human only stretchs out an arm, our shape model-based prediction prevents unnecessary reactive motions, which results a better robot trajectory than the prediction using simple extrapolations.
           (c) A collision-free computed trajectory that avoids collisions with the obstacle and the environment.
           (d) The robot adjusts its speed or waits if there is no feasible path to the goal position due to the dynamic obstacles.}
\label{fig:experiment1}
\end{figure}

Table~\ref{table:performance} presents the complexity of the benchmarks and the performance of our planning results. 
Our first benchmark tests our shape model-based environment belief state prediction.
When a human is dashing onto the robot at a fast speed, the robot is aware of the potential collision with the predicted future human position and changes its trajectory (Fig.~\ref{fig:experiment1}(a)). 
However, if a human in standing only stretchs out an arm toward the robot, even if the velocity of the arm is fast, the model-based prediction prevents unnecessary reactive motions, which is different from the prediction models with constant velocity or acceleration extrapolations (Fig.~\ref{fig:experiment1}(b)).

The second benchmark set shows our planner's collision-free path computation in the space-time domain. The planner computes a collision-free trajectory that avoids collision with the obstacles and the environments (Fig.~\ref{fig:experiment1}(c)). If there is no feasible solution due to dynamic obstacles, the planner adjusts its speed or waits until it finds a solution (Fig.~\ref{fig:experiment1}(d)).


\begin{figure}[ht]
  \centering
  \subfloat
  {
    \includegraphics[width=0.3\linewidth]{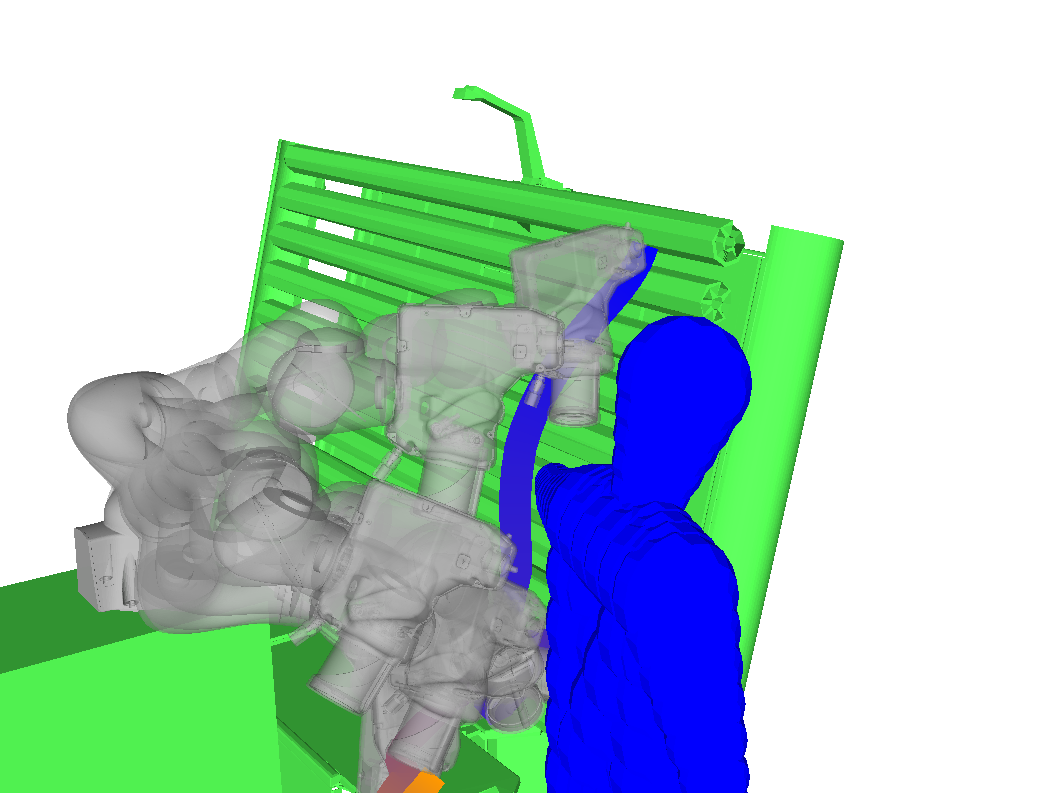}
  }
  \subfloat
  {
    \includegraphics[width=0.3\linewidth]{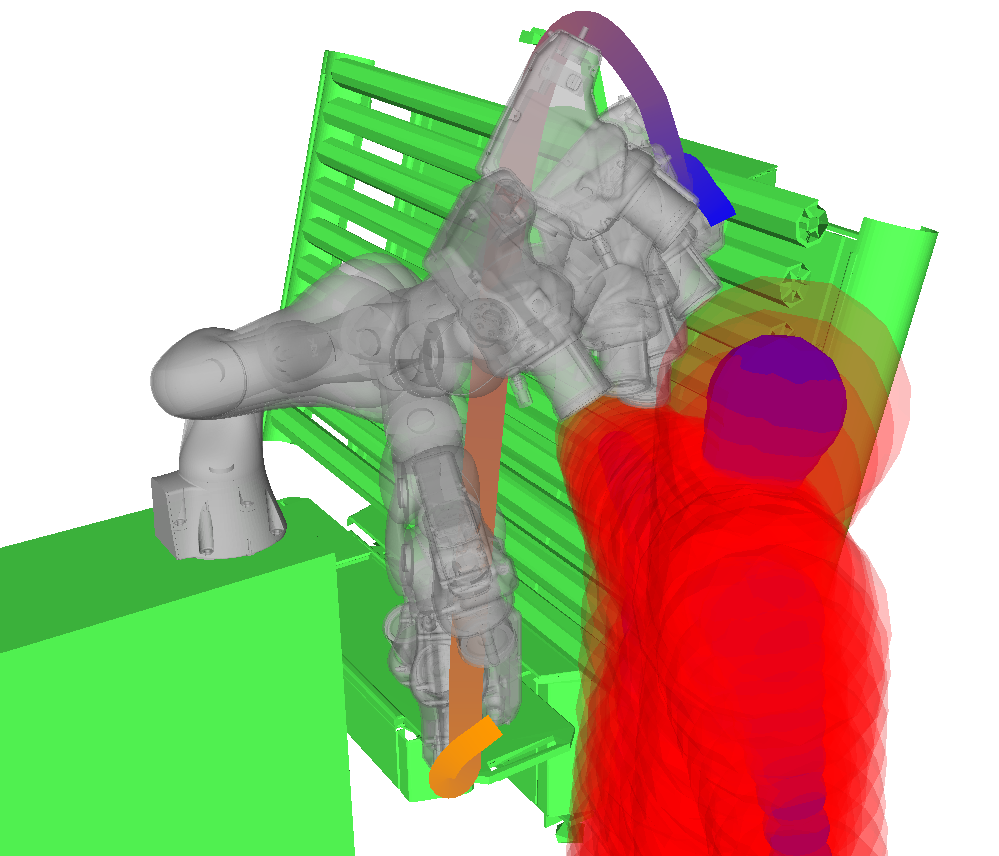}
  }
  \subfloat
  {
    \includegraphics[width=0.3\linewidth]{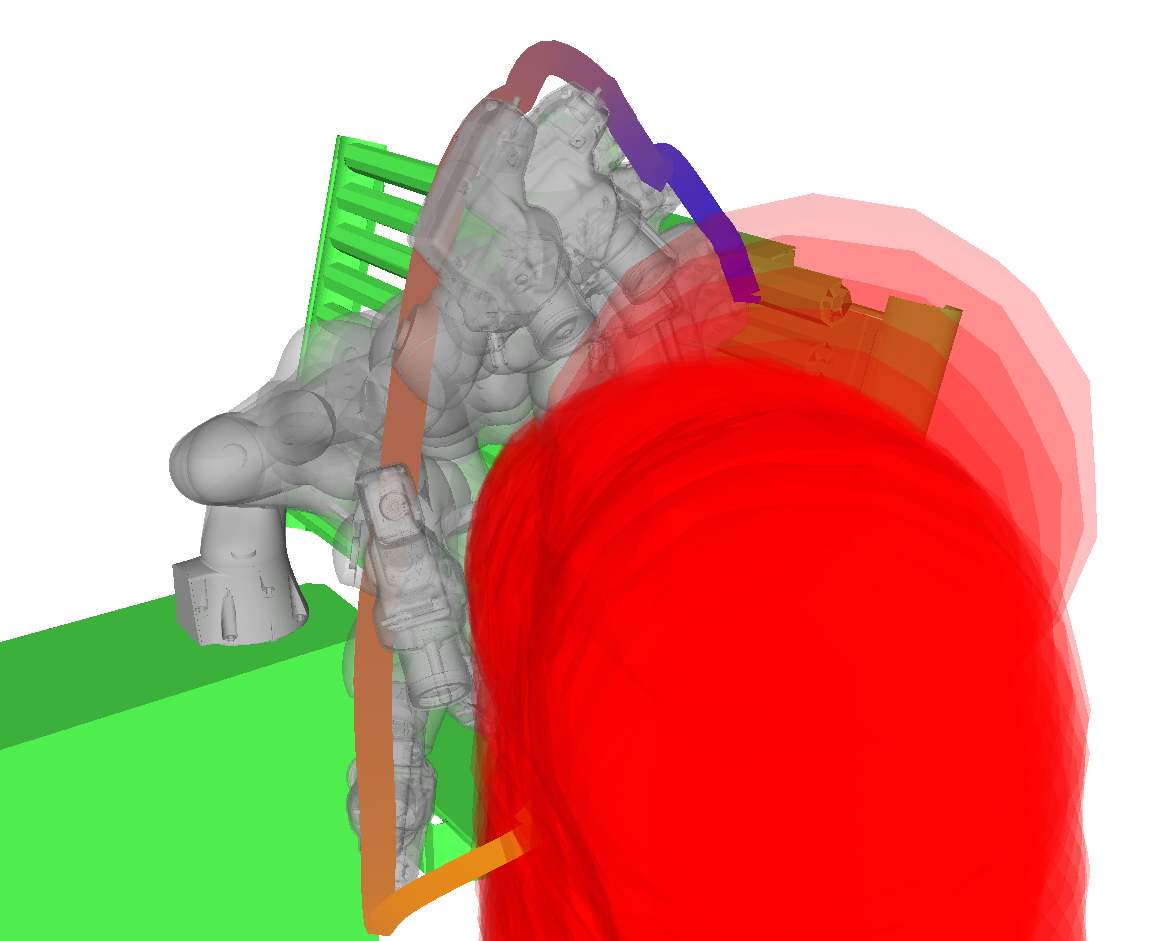}
  }
  \caption{{\bf Robot Trajectory with Different Confidence and Noise Levels:} 
           (a) A trajectory with exact collision checking for zero-noise obstacles.
           (b) A trajectory with $\delta_{CL}=0.95$ and $\mathbf v_t=0.005 I_{3 \times 3}$.
           (c) A trajectory with $\delta_{CL}=0.99$ and $\mathbf v_t=0.05 I_{3 \times 3}$.
           }
\label{fig:experiment2}           
\end{figure}

Fig.~\ref{fig:experiment2} shows a robot trajectory with different confidence levels and sensor noises. If the obstacle states are assumed as exact, the robot can follow the shortest and smoothest trajectory that is close to the obstacle (Fig.~\ref{fig:experiment2}(a)). However, as the noise of the environment state or expected confidence level becomes higher, the computed robot trajectories become longer and less smooth to avoid potential collision with the obstacles (Fig.~\ref{fig:experiment2}(b)-(c)).

\subsection{Comparison with Other Algorithms}

\begin{table}[ht]
\centering
\begin{tabular}{|c|c|c|c|}
\hline
Algorithms &  \begin{tabular}[x]{@{}c@{}}Number of\\Collisions\end{tabular} & \begin{tabular}[x]{@{}c@{}}Trajectory\\Duration (sec)\end{tabular}   &  \begin{tabular}[x]{@{}c@{}}Trajectory\\Length (m)\end{tabular}\\ \hline
\begin{tabular}[x]{@{}c@{}}Enlarged bounding\\ volumes~\cite{van2012lqg,Park:2012:ICAPS}\end{tabular} &  0.02 &  10.47 & 2.32\\ \hline
\begin{tabular}[x]{@{}c@{}}Approximation using\\ the center point PDF~\cite{du2011probabilistic}\end{tabular}  & 0.43 & 6.62 & 1.52\\ \hline
Our approach &  0.03 & 7.16 & 1.87 \\ \hline
\end{tabular}
\caption{{\bf Planning results of different probabilistic collision detection algorithms:} Our probabilistic collision detection approach shows a high safety as the approach using enlarged bounding volumes, while computes efficient trajectories.}
\label{table:performance2}
\end{table}

In order to compare our algorithm with other probabilistic collision detection algorithms, we plan trajectories using the different probabilistic collision detection algorithms. 
We choose different initial and goal configurations for each trials, and compute trajectories with $\delta_{CL}=0.99$. The trajectory durations are initialized to 5 seconds.
We applied the same recorded human motion that stretches an arm that blocks the initial robot trajectory to the planning, but each trial has a different small perturbation of the human obstacle position that corresponds to the environment uncertainties.
We measure the number of collisions, as well as the durations and lengths of the computed trajectories for planners with three different probabilistic collision detection algorithms. The averages of 100 trials are shown in Table~\ref{table:performance2}. The enlarged bounding volumes have less collisions, but the durations and lengths of the computed trajectories are longer than other approaches, since the overestimated collision probability makes the planner to compute trajectories which are unnecessarily far apart from the obstacles, or to wait when there is a feasible trajectory.
On the other hand, the approximating approach that uses the probability of the object center point underestimate the collision probability and causes several collisions in the planned trajectories.
Our approach shows a similar safety with the approach using enlarged bounding volumes, while it also computes efficient trajectories that have shorter trajectory durations and lengths.

\section{Conclusions and Future Work}
We present a novel algorithm for trajectory planning for high-DOF robots in dynamic, uncertain environments. This include new methods for belief space estimation and probabilistic collision detection. Our approach is quite fast, and works well in our simulated results where it can compute collision-free paths with high confidence level.  Our probabilistic collision detection computes tighter upper bounds of the collision probability as compared to prior approaches. We highlight the performance of our planner on different benchmarks with human obstacles for two robot models. To the best of our knowledge, that can handle high-DOF robots in dynamic environment with imperfect obstacle representations.

Our approach has some limitations. Some of the assumptions used in belief space estimation in terms of Gaussian distribution and Kalman filter may not hold. Moreover, we may have a pre-defined shape representation of the obstacle. The trajectory optimization may get stuck at a local minima and may not converge to a global optimal solution.
 There are many avenues for future work. Our approach only takes into account the imperfect information about the moving obstacles. In particular, we assume that a good point-cloud sample of the obstacles is given for probabilistic collision checking.  Our current approach does not account for control errors or sensor errors, which are rather common with the controllers and sensors.

\bibliographystyle{splncs03}
\bibliography{rss16}

\begin{thebibliography}{10}
\providecommand{\url}[1]{\texttt{#1}}
\providecommand{\urlprefix}{URL }

\bibitem{bae2009closed}
Bae, K.H., Belton, D., Lichti, D.D.: A closed-form expression of the positional
  uncertainty for 3d point clouds. Pattern Analysis and Machine Intelligence,
  IEEE Transactions on  31(4),  577--590 (2009)

\bibitem{bai2015intention}
Bai, H., Cai, S., Ye, N., Hsu, D., Lee, W.S.: Intention-aware online pomdp
  planning for autonomous driving in a crowd. In: Robotics and Automation
  (ICRA), 2015 IEEE International Conference on. pp. 454--460. IEEE (2015)

\bibitem{bandyopadhyay2009motion}
Bandyopadhyay, T., Rong, N., Ang, M., Hsu, D., Lee, W.S.: Motion planning for
  people tracking in uncertain and dynamic environments. In: Workshop on People
  Detection and Tracking, IEEE International Conference on Robotics and
  Automation (2009)

\bibitem{van2012lqg}
Van~den Berg, J., Wilkie, D., Guy, S.J., Niethammer, M., Manocha, D.:
  {LQG}-{O}bstacles: Feedback control with collision avoidance for mobile
  robots with motion and sensing uncertainty. In: Robotics and Automation
  (ICRA), 2012 IEEE International Conference on. pp. 346--353. IEEE (2012)

\bibitem{blackmore2006probabilistic}
Blackmore, L.: A probabilistic particle control approach to optimal, robust
  predictive control. In: Proceedings of the AIAA Guidance, Navigation and
  Control Conference. No.~10 (2006)

\bibitem{bry2011rapidly}
Bry, A., Roy, N.: Rapidly-exploring random belief trees for motion planning
  under uncertainty. In: Robotics and Automation (ICRA), 2011 IEEE
  International Conference on. pp. 723--730. IEEE (2011)

\bibitem{charrow2015information}
Charrow, B., Kahn, G., Patil, S., Liu, S., Goldberg, K., Abbeel, P., Michael,
  N., Kumar, V.: Information-theoretic planning with trajectory optimization
  for dense 3d mapping. In: Proceedings of Robotics: Science and Systems (2015)

\bibitem{du2011probabilistic}
Du~Toit, N.E., Burdick, J.W.: Probabilistic collision checking with chance
  constraints. Robotics, IEEE Transactions on  27(4),  809--815 (2011)

\bibitem{du2012robot}
Du~Toit, N.E., Burdick, J.W.: Robot motion planning in dynamic, uncertain
  environments. Robotics, IEEE Transactions on  28(1),  101--115 (2012)

\bibitem{Fiorini:1998}
Fiorini, P., Shiller, Z.: Motion planning in dynamic environments using
  velocity obstacles. International Journal of Robotics Research  17(7),
  760--772 (1998)

\bibitem{ganapathi12realtime}
Ganapathi, V., Plagemann, C., Koller, D., Thrun, S.: Real-time human pose
  tracking from range data. Springer (2012)

\bibitem{groetsch1984theory}
Groetsch, C.W.: The theory of tikhonov regularization for fredholm equations of
  the first kind  (1984)

\bibitem{guibas2010bounded}
Guibas, L.J., Hsu, D., Kurniawati, H., Rehman, E.: Bounded uncertainty roadmaps
  for path planning. In: Algorithmic Foundation of Robotics VIII, pp. 199--215.
  Springer (2010)

\bibitem{haschke2008line}
Haschke, R., Weitnauer, E., Ritter, H.: On-line planning of time-optimal,
  jerk-limited trajectories. In: Intelligent Robots and Systems, 2008. IROS
  2008. IEEE/RSJ International Conference on. pp. 3248--3253. IEEE (2008)

\bibitem{Hauser:safety}
Hauser, K.: On responsiveness, safety, and completeness in real-time motion
  planning. Autonomous Robots  32(1),  35--48 (2012)

\bibitem{David:2002}
Hsu, D., Kindel, R., Latombe, J.C., Rock, S.: Randomized kinodynamic motion
  planning with moving obstacles. International Journal of Robotics Research
  21(3),  233--255 (March 2002)

\bibitem{kaelbling1998planning}
Kaelbling, L.P., Littman, M.L., Cassandra, A.R.: Planning and acting in
  partially observable stochastic domains. Artificial intelligence  101(1),
  99--134 (1998)

\bibitem{kahn2015active}
Kahn, G., Sujan, P., Patil, S., Bopardikar, S., Ryde, J., Goldberg, K., Abbeel,
  P.: Active exploration using trajectory optimization for robotic grasping in
  the presence of occlusions. In: 2015 IEEE International Conference on
  Robotics and Automation (ICRA). pp. 4783--4790. IEEE (2015)

\bibitem{Koenig:2003:PBP}
Koenig, S., Tovey, C., Smirnov, Y.: Performance bounds for planning in unknown
  terrain. Artificial Intelligence  147(1-2),  253--279 (July 2003)

\bibitem{kroger2010online}
Kroger, T., Wahl, F.M.: Online trajectory generation: Basic concepts for
  instantaneous reactions to unforeseen events. Robotics, IEEE Transactions on
  26(1),  94--111 (2010)

\bibitem{kurniawati2013online}
Kurniawati, H., Yadav, V.: An online pomdp solver for uncertainty planning in
  dynamic environment. ISRR (2013)

\bibitem{lambert2008fast}
Lambert, A., Gruyer, D., Pierre, G.S.: A fast monte carlo algorithm for
  collision probability estimation. In: Control, Automation, Robotics and
  Vision, 2008. ICARCV 2008. 10th International Conference on. pp. 406--411.
  IEEE (2008)

\bibitem{LaValle:2006}
LaValle, S.M.: Planning Algorithms. Cambridge University Press (2006)

\bibitem{lee2013sigma}
Lee, A., Duan, Y., Patil, S., Schulman, J., McCarthy, Z., van~den Berg, J.,
  Goldberg, K., Abbeel, P.: Sigma hulls for gaussian belief space planning for
  imprecise articulated robots amid obstacles. In: Intelligent Robots and
  Systems (IROS), 2013 IEEE/RSJ International Conference on. pp. 5660--5667.
  IEEE (2013)

\bibitem{lee2013gpu}
Lee, T., Kim, Y.J.: Gpu-based motion planning under uncertainties using pomdp.
  In: Robotics and Automation (ICRA), 2013 IEEE International Conference on.
  pp. 4576--4581. IEEE (2013)

\bibitem{lepetit2005randomized}
Lepetit, V., Lagger, P., Fua, P.: Randomized trees for real-time keypoint
  recognition. In: Computer Vision and Pattern Recognition, 2005. CVPR 2005.
  IEEE Computer Society Conference on. vol.~2, pp. 775--781. IEEE (2005)

\bibitem{leung2006planning}
Leung, C., Huang, S., Kwok, N., Dissanayake, G.: Planning under uncertainty
  using model predictive control for information gathering. Robotics and
  Autonomous Systems  54(11),  898--910 (2006)

\bibitem{Likhachev:2009}
Likhachev, M., Ferguson, D.: Planning long dynamically feasible maneuvers for
  autonomous vehicles. International Journal of Robotics Research  28(8),
  933--945 (August 2009)

\bibitem{Likhachev05anytimedynamic}
Likhachev, M., Ferguson, D., Gordon, G., Stentz, A., Thrun, S.: Anytime dynamic
  {A}*: An anytime, replanning algorithm. In: Proceedings of the International
  Conference on Automated Planning and Scheduling (2005)

\bibitem{missiuro2006adapting}
Missiuro, P.E., Roy, N.: Adapting probabilistic roadmaps to handle uncertain
  maps. In: Robotics and Automation, 2006. ICRA 2006. Proceedings 2006 IEEE
  International Conference on. pp. 1261--1267. IEEE (2006)

\bibitem{nguyen2012modeling}
Nguyen, C.V., Izadi, S., Lovell, D.: Modeling kinect sensor noise for improved
  3d reconstruction and tracking. In: 3D Imaging, Modeling, Processing,
  Visualization and Transmission (3DIMPVT), 2012 Second International
  Conference on. pp. 524--530. IEEE (2012)

\bibitem{pan2011probabilistic}
Pan, J., Chitta, S., Manocha, D.: Probabilistic collision detection between
  noisy point clouds using robust classification. In: International Symposium
  on Robotics Research (ISRR) (2011)

\bibitem{pan2013real}
Pan, J., {\c{S}}ucan, I.A., Chitta, S., Manocha, D.: Real-time collision
  detection and distance computation on point cloud sensor data. In: Robotics
  and Automation (ICRA), 2013 IEEE International Conference on. pp. 3593--3599.
  IEEE (2013)

\bibitem{papadimitriou1987complexity}
Papadimitriou, C.H., Tsitsiklis, J.N.: The complexity of markov decision
  processes. Mathematics of operations research  12(3),  441--450 (1987)

\bibitem{Park:2012:ICAPS}
Park, C., Pan, J., Manocha, D.: {ITOMP}: Incremental trajectory optimization
  for real-time replanning in dynamic environments. In: Proceedings of
  International Conference on Automated Planning and Scheduling (2012)

\bibitem{SMP:2005}
Petti, S., Fraichard, T.: Safe motion planning in dynamic environments. In:
  Proceedings of IEEE/RSJ International Conference on Intelligent Robots and
  Systems. pp. 2210--2215 (2005)

\bibitem{plagemann2010real}
Plagemann, C., Ganapathi, V., Koller, D., Thrun, S.: Real-time identification
  and localization of body parts from depth images. In: Robotics and Automation
  (ICRA), 2010 IEEE International Conference on. pp. 3108--3113. IEEE (2010)

\bibitem{platt2010belief}
Platt~Jr, R., Tedrake, R., Kaelbling, L., Lozano-Perez, T.: Belief space
  planning assuming maximum likelihood observations  (2010)

\bibitem{shani2010evaluating}
Shani, G.: Evaluating point-based pomdp solvers on multicore machines. Systems,
  Man, and Cybernetics, Part B: Cybernetics, IEEE Transactions on  40(4),
  1062--1074 (2010)

\bibitem{shotton2013real}
Shotton, J., Sharp, T., Kipman, A., Fitzgibbon, A., Finocchio, M., Blake, A.,
  Cook, M., Moore, R.: Real-time human pose recognition in parts from single
  depth images. Communications of the ACM  56(1),  116--124 (2013)

\bibitem{silver2010monte}
Silver, D., Veness, J.: Monte-carlo planning in large pomdps. In: Advances in
  neural information processing systems. pp. 2164--2172 (2010)

\bibitem{somani2013despot}
Somani, A., Ye, N., Hsu, D., Lee, W.S.: Despot: Online pomdp planning with
  regularization. In: Advances In Neural Information Processing Systems. pp.
  1772--1780 (2013)

\bibitem{sun2015stochastic}
Sun, W., van~den Berg, J., Alterovitz, R.: Stochastic extended lqr:
  Optimization-based motion planning under uncertainty. In: Algorithmic
  Foundations of Robotics XI, pp. 609--626. Springer (2015)

\bibitem{Zucker:IJRR:2012}
Zucker, M., Ratliff, N., Dragan, A.D., Pivtoraiko, M., Klingensmith, M.,
  Dellin, C.M., Bagnell, J.A., Srinivasa, S.S.: {CHOMP}: Covariant hamiltonian
  optimization for motion planning. International Journal of Robotics Research
  (2012)

\end{thebibliography}

\end{document}